\documentclass[11pt]{article}

\usepackage[top=1in,bottom=1in,left=1in,right=1in]{geometry}
\usepackage[utf8]{inputenc} % allow utf-8 input
\usepackage[T1]{fontenc}    % use 8-bit T1 fonts
\usepackage{booktabs}       % professional-quality tables
\usepackage{amsfonts}       % blackboard math sy
\usepackage{amsmath} 
\usepackage{dsfont}
\usepackage{amssymb}
\usepackage{mathtools}
\usepackage{amsthm}
\usepackage{nicefrac}
\usepackage{mathrsfs} 
\usepackage{nicematrix}

\usepackage{subfig}
\usepackage{xr}
\usepackage{tikz}
\usepackage{pgfplots}

\usepackage{adjustbox}
\usepackage{bm, braket}

\usepackage{natbib}
\usepackage{doi}
\usepackage{hyperref}

\usepackage{extarrows} 
\usepackage{enumitem}

\makeatletter
\DeclareFontEncoding{LS1}{}{}
\DeclareFontSubstitution{LS1}{stix}{m}{n}
\DeclareMathAlphabet{\mathscr}{LS1}{stixscr}{m}{n}
\makeatother
\DeclareMathOperator{\sign}{sign}
\DeclareMathOperator{\Tr}{Tr}

\title{One-Bit Quantization for Random Features Models}

\author{ Danil Akhtiamov \thanks{Equal Contribution} \thanks{Department of Computing and Mathematical Sciences, California Institute of Technology} \and  Reza Ghane \footnotemark[1] \thanks{Department of Electrical Engineering, 
	California Institute of Technology} \and Babak Hassibi \footnotemark[2] \,\footnotemark[3] }

\renewcommand{\bibname}{References}
\renewcommand{\bibsection}{\subsubsection*{\bibname}}

\newcommand{\bSigma}{\bm{\Sigma}}

\newcommand{\polylog}{\text{polylog}}

\newcommand{\bmu}{\bm{\mu}}

\newcommand{\tlalpha}{\tilde{\alpha}}

\newcommand{\tleta}{\tilde{\eta}}

\newcommand{\bG}{\mathbf{G}}

\newcommand{\bI}{\mathbf{I}}

\newcommand{\bW}{\mathbf{W}}
\newcommand{\bX}{\mathbf{X}}
\newcommand{\bY}{\mathbf{Y}}

\newcommand{\ba}{\mathbf{a}}
\newcommand{\bb}{\mathbf{b}}

\newcommand{\bg}{\mathbf{g}}
\newcommand{\bh}{\mathbf{h}}

\newcommand{\bs}{\mathbf{s}}
\newcommand{\bu}{\mathbf{u}}
\newcommand{\bv}{\mathbf{v}}
\newcommand{\bw}{\mathbf{w}}
\newcommand{\bx}{\mathbf{x}}
\newcommand{\by}{\mathbf{y}}
\newcommand{\bz}{\mathbf{z}}

\newcommand{\bbE}{\mathbb{E}}
\newcommand{\bbR}{\mathbb{R}}

\newcommand{\bbP}{\mathbb{P}}

\newcommand{\calN}{\mathcal{N}}
\newcommand{\calM}{\mathcal{M}}

\newcommand{\tbSigma}{\tilde{\bSigma}}

\newcommand{\tlbG}{\tilde{\bG}}

\newcommand{\tlbg}{\tilde{\bg}}
\newcommand{\tlbh}{\tilde{\bh}}
\newcommand{\tlbu}{\tilde{\bu}}

\newcommand{\bzero}{\mathbf{0}}

\newcommand{\tr}{\text{Tr}}

\newcommand{\rarrowp}{\xrightarrow[]{\bbP}}

\theoremstyle{plain}
\newtheorem{theorem}{Theorem}
\newtheorem{remark}{Remark}
\newtheorem{lemma}{Lemma}

\newtheorem{definition}{Definition}
\newtheorem{assumptions}{Assumptions}

\begin{document}
\maketitle

\begin{abstract}

Recent advances in neural networks have led to significant computational and memory demands, spurring interest in one-bit weight compression to enable efficient inference on resource-constrained devices. However, the theoretical underpinnings of such compression remain poorly understood. We address this gap by analyzing one-bit quantization in the Random Features model, a simplified framework that corresponds to neural networks with random representations. We prove that, asymptotically,  quantizing weights of all layers except the last incurs no loss
in generalization error, compared to the full precision random features model. Our findings offer theoretical insights into neural network compression. We also demonstrate empirically that one-bit quantization leads to significant inference speed ups for the Random Features models even on a laptop GPU, confirming the practical benefits of our work. Additionally, we provide an asymptotically precise characterization of the generalization error for Random Features with an arbitrary number of layers. To the best of our knowledge, our analysis yields more general results than all previous works in the related literature.

 % One-bit compression of neural networks has been a topic of active research recently due to the exorbitant inference and memory costs of modern neural networks. Since little is understood theoretically about one-bit compression of neural networks, we propose to explore this topic a simplified model known as {\it the Random Features model} in the literature. For Random Features Models, we show that, provided they are sufficiently large, one can quantize the weights of all the layers up to the last one {\it without any loss in performance}. To prove this rigorously, we deploy the Lindenberg approach and show a Gaussian Universality result. After that, we proceed to precisely quantify the loss of performance resulting after compression of the last layer. 

% Furthermore, we precisely quantify the performance degradation resulting from quantizing the last layer, providing explicit expressions for the errors of the model with and without quantization of the last layers. 
 
\end{abstract}

\section{Introduction}

The success of deep neural networks in tasks such as image recognition, natural language processing, and reinforcement learning has come at the cost of escalating computational and memory requirements. Modern models, often comprised of billions of parameters, demand significant resources for training and inference, rendering them impractical for deployment on resource-constrained devices like mobile phones, embedded systems, or IoT devices. To address this challenge, weight quantization---reducing the precision of neural network weights---has emerged as a promising technique to lower memory footprint and accelerate inference. In particular, one-bit quantization, which restricts weights to \(\{+1, -1\}\), offers extreme compression (e.g., \(\sim 32\times\) memory reduction for 32-bit floats) and enables efficient hardware implementations using bitwise operations. Various works have explored the possibility of network quantization in the recent years. In particular, for Large Language Models (LLMs), some post-training have been able to reduce the model size via fine-tuning. Examples of such approach include GPTQ \cite{frantar2022gptq} which can quantize a 175 billion GPT model to 4 bits and QuIP which \cite{chee2023quip} compresses Llama 2 70B to 2 and 3 bits.  Furthermore, quantization-aware training approaches, such as Bitnet \cite{wang2023bitnet}, Bitnet 1.58b \cite{ma2024era}, have been able to achieve one-bit language models with comparable performance to the models from the same weight class. For a recent survey on efficient LLMs we refer to \cite{xu2024survey}. Such results are desirable as they pave the way for bringing foundational models to edge devices by reducing memory requirements and reducing the inference time. However, while the aforementioned empirical approaches have demonstrated practical success, the theoretical foundations of one-bit quantization remain underexplored, limiting our ability to predict its performance and design improved training algorithms.

% mpirical methods like BinaryConnect~\cite{courbariaux2015binaryconnect} and XNOR-Net~\cite{rastegari2016xnor} 

This paper investigates the generalization properties of one-bit quantization in the Random Features model, a simplified framework that captures key properties of wide neural networks while being amenable to rigorous analysis. Introduced by Rahimi and Recht \cite{rahimi2007random}, the Random Features model approximates kernel methods and corresponds to the infinite-width limit of neural networks under certain conditions \cite{jacot2018neural}. By studying quantization in this model, we aim to uncover fundamental principles that govern the trade-offs between compression and performance in neural networks, leading to memory savings and inference speed-ups. Our main contributions are twofold:
\\

\begin{enumerate}
    \item \textbf{Lossless Quantization of Hidden Layers:} We prove that, for sufficiently wide Random Features models, quantizing the weights of all layers except the last to one bit incurs no loss in generalization error. This surprising result is established via a Gaussian Universality (GU) and Gaussian Equivalence (GE): GU implies that the test error of the linear model trained on the outputs of the random features matches the test error of the linear model trained on Gaussians with the same covariance; GE implies that the covariance of the random features, and the necessary characteristics of the latter covariance, are the same for the quantized and unquantized weights.
    \item \textbf{Precise Characterization of the Test error of Deep Random Features model:} In the proportional regime, we rigorously characterize the generalization error of Random Features model with quantized weights with multiple layers and express the generalization error in terms of a few scalar variables.
\end{enumerate}

The rest of the paper is organized as follows: Section ~\ref{sec:preliminaries} introduces the Random Features model and our notation, reviews Stochastic Mirror Descent and its implicit regularization properties, and presents the Gaussian Universality and Gaussian Equivalence principles that form the foundation of our analysis. Section ~\ref{sec: related} reviews related work on Random Features models and Gaussian universality. Section ~\ref{sec:results} details our main theorems on quantization, Section ~\ref{sec: approach} discusses our approach and contributions, and Section ~\ref{sec:numerics} presents numerical validations of our theoretical findings.

\section{Preliminaries}
\label{sec:preliminaries}
Throughout the paper we use bold letters for vectors and matrices.
\subsection{Lipschitz Concentration Property}

The following definition will be necessary for presenting our main result.

\begin{definition}[Lipschitz Concentration Property]\label{def: LCP}
A random vector $\bz \in \mathbb{R}^d$ satisfies the Lipschitz Concentration Property (LCP) with parameter $\sigma$ if for any $L$-Lipschitz function $f: \mathbb{R}^d \to \mathbb{R}$, the random variable $f(\bz) - \mathbb{E}[f(\bz)]$ is subgaussian with parameter $L\sigma$. That is, for all $t > 0$:
\[
\mathbb{P}\Bigl(\Bigl|f(\bz) - \mathbb{E}[f(\bz)]\Bigr| \geq t\Bigr) \leq 2\exp\left(-\frac{t^2}{2L^2\sigma^2}\right)
\]
\end{definition}

\subsection{Problem Setting}
\label{subs: notation}
We consider a Random Features model defined as follows:
\begin{itemize}
    \item \textbf{Input and Output}: Let \(\mathbf{x} \in \mathbb{R}^d\) denote the input vector, and \(y \in \mathbb{R}\)  denote the target. The dataset consists of \(n\) samples \((\mathbf{x}_i, y_i)\).  Furthermore, we assume that the data $\bx$ satisfies Definition \ref{def: LCP} with $\sigma^2 = O\left(\frac{1}{d}\right)$. We also assume that $\bx$ is centered, i.e. $\bbE \bx = 0$. Denote $\bSigma = \bbE \bx\bx^T$. Then the assumptions made in this bullet imply that   $$\kappa(\bSigma) = \frac{\sigma_1}{\sigma_{m}} = O(1) \text{ and }\tr(\bSigma) = O(1),$$ where $\sigma_1, \dots, \sigma_m$ are the eigenvalues of $\bSigma$ in the decreasing order. In other words, the matrix $\bSigma$ is well-conditioned and normalized so that the norms of the inputs $\|\bx\|_2 = O(1)$.   
    \item \textbf{Model Architecture}: The Random Features model is a neural network with $L$ hidden layers and an activation function $\phi$:
    \begin{itemize}
        \item \textbf{Hidden layers}: The input to each hidden layer is mapped to a feature vector via random weights \(\mathbf{W} \in \mathbb{R}^{d_\ell \times d_{\ell - 1}}\), where \( d_{\ell - 1}\) is the number of input features to layer $\ell$ and $d_\ell$ is the number of output features. Each entry \(\bW_{ij} \sim \mathcal{N}\left(0, \frac{1}{d_{\ell-1}}\right)\) for the non-quantized model and \(\bW_{ij} \sim \frac{1}{\sqrt{d_{\ell-1}}}\mathbf{Unif}(-1,+1)\) for the quantized model. Note that the coefficient $\frac{1}{d_{\ell-1}}$ is necessary to ensure that the quantized model has the same second order statistics as the non-quantized model.  The map for the $\ell$-th hidden layer is \(\phi_{\ell}(\mathbf{\bx^{\ell-1}}) = \phi(\bW^{(\ell)}\bx^{(\ell - 1)})\) and $\bx^{(0)} = \bx$ is the input distribution.
        \item \textbf{Last layer}: The output is a linear combination of features, $$f(\mathbf{x}, \ba, \bW^{1}, \dots, \bW^{L}) = \mathbf{a}^\top \bx^{(L)},$$ where $\mathbf{a} \in \mathbb{R}^{d_L}$ is the output layer weights.
    \end{itemize}
    \item \textbf{Quantization}: One-bit quantization maps weights in the hidden layer $\ell$ to \(\frac{1}{\sqrt{d_{\ell-1}}}\{+1, -1\}\) by preserving the normalization and taking the sign of each entry. Since the hidden layers for the non-quantized model are gaussian, this means that for the quantized model. Note that we quantize only the hidden layers and do not quantize the last layer $\ba$ and the data $\bx$. As a motivation, note that the majority of the memory is taken up by the weights in the hidden layers for our model and, therefore, we reduce memory requirements almost by a factor of $32$ assuming the non-quantized model has $32$-bit weights. Moreover, we demonstrate empirically that the presented scheme leads to almost 4X inference speed ups for sufficiently wide hidden layers. 
    
    \item \textbf{Training procedure}: We assume that the last layer is trained to minimize an arbitrary differentiable convex loss function satisfying $\min_t \mathcal{L}(t) = \mathcal{L}(0) = 0$, i.e. the following optimization is performed, via {\it Stochastic Mirror Descent}: $$\min_{\ba} \sum_{i = 1}^n  \mathcal{L}\left(y_i - f(\mathbf{x}_i, \ba, \bW^{1}, \dots, \bW^{L})\right)$$ Moreover, we assume that the model is over-parametrized, i.e. the number of parameters in the last layer exceeds the number of data points. Over-parametrization is a common assumption in modern machine learning. 
    \item \textbf{Ground truth}: We assume that the labels are generated according to 
    \begin{align}\label{eq: labels}
    y = f(\mathbf{x}, \ba_*, \bW^{1}, \dots, \bW^{L})
    \end{align}
    here, $\ba_*$ is a ground truth parameter that we take to be $\ba_* \sim \mathcal{N}(0, \frac{\bI}{d_L})$, as it is natural to assume the ground truth is a "generic" vector. 
    \item \textbf{Performance Metric}: We measure performance of a trained model via the MSE loss $$\mathbb{E}_{\bx}[(f(\mathbf{x}_i, \ba, \bW^{1}, \dots, \bW^{L})- y)^2]$$
    \item \textbf{Scaling}: We assume $d \to \infty$ and the hidden layer dimensions grow proportionally, i.e. $\gamma_\ell = \frac{d_\ell}{d}$ is constant for $\ell = 1,\dots, L$. 
\end{itemize}

\subsection{Stochastic Mirror Descent and Implicit Regularization}

% Stochastic Mirror Descent (SMD) generalizes stochastic gradient descent by employing a strongly convex, differentiable mirror map $\psi$. For a loss function $\mathcal{L}(\mathbf{\bw}; \bx, y)$ and data $\{(\bx_1, y_1), \dots, (\bx_n, y_n)\}$, the SMD update at step $t$ is
% \[
% \nabla\psi(\mathbf{w}_{t+1}) = \nabla\psi(\mathbf{w}_t) - \eta \nabla \sum_{i=1}^n\mathcal{L}(\mathbf{\bw}_t; \bx_i, y_i)
% \] 
% Note that taking $\psi(\bw) = \frac{\|\bw\|^2}{2}$ corresponds to the usual gradient descent:
% \[
% \mathbf{w}_{t+1} = \mathbf{w}_t - \eta \nabla \sum_{i=1}^n\mathcal{L}(\mathbf{\bw}_t; \bx_i, y_i)
% \] 

% Implicit regularization refers to the phenomenon where optimization algorithms naturally favor solutions minimizing certain characteristics of the weights without explicit regularization terms in the objective function. In overparameterized linear models, where the number of parameters exceeds the number of samples (\(d > n\)), SMD exhibits a crucial implicit bias property \cite{azizan2021stochastic}: among all interpolating solutions, i.e. solutions satisfying
% $\bX\bw = \by$, it chooses the solution that minimizes the value of $D_\psi(\bw, \bw_0)$. In other words, the following holds:
% $$\lim_{t \to \infty} \mathbf{w}_t = \arg\min_{\mathbf{w}} D\psi(\mathbf{w}, \bw_0) \text{ subject to } \bX\bw = \by$$

Stochastic Mirror Descent (SMD) generalizes Stochastic Gradient Descent (SGD) by employing a strictly convex, differentiable mirror map $\psi$. For a loss function $\mathcal{L}(\mathbf{w}; \mathbf{x}, y)$ and data $\{(\mathbf{x}_1, y_1), \dots, (\mathbf{x}_n, y_n)\}$, the SMD update at step $t$ is
\[\nabla\psi(\mathbf{w}_{t+1}) = \nabla\psi(\mathbf{w}_t) - \eta \nabla \sum_{i=1}^n\mathcal{L}(\mathbf{w}_t; \mathbf{x}_i, y_i),\] 
Note that taking $\psi(\mathbf{w}) = \frac{\|\mathbf{w}\|^2}{2}$ corresponds to the usual gradient descent:
\[\mathbf{w}_{t+1} = \mathbf{w}_t - \eta \nabla \sum_{i=1}^n\mathcal{L}(\mathbf{w}_t; \mathbf{x}_i, y_i).\] 
Implicit regularization refers to the phenomenon where optimization algorithms naturally favor solutions minimizing certain characteristics of the weights without explicit regularization terms in the objective function. In overparameterized linear models, where the number of parameters exceeds the number of samples ($d > n$), SMD exhibits a crucial implicit bias property \cite{azizan2021stochastic}: among all interpolating solutions, i.e., solutions satisfying $\mathbf{X}\mathbf{w} = \mathbf{y}$, it chooses the solution that minimizes the Bregman divergence from the initialization $\mathbf{w}_0$. In other words, the following holds:
\begin{align}\label{eq: implc_smd_breg}
\lim_{t \to \infty} \mathbf{w}_t = \arg\min_{\mathbf{w}} D_\psi(\mathbf{w}, \mathbf{w}_0) \text{ subject to } \mathbf{X}\mathbf{w} = \mathbf{y}
\end{align}
where the Bregman divergence is defined as $$D_\psi(\mathbf{w}, \mathbf{w}') = \psi(\mathbf{w}) - \psi(\mathbf{w}') - {\nabla \psi(\mathbf{w}')}^T(\mathbf{w} - \mathbf{w}')$$
For the gradient descent with initialization $\bw_0 \approx 0$, \eqref{eq: implc_smd_breg} takes the following simple form:
\begin{align}\label{eq: implc_smd}
\lim_{t \to \infty} \mathbf{w}_t = \arg\min_{\mathbf{w}} \|\bw\|_2^2 \text{ subject to } \mathbf{X}\mathbf{w} = \mathbf{y}
\end{align}

\subsection{Gaussian Universality}

Theorem \ref{thm: univ}, presented in \cite{ghane2024universality}, establishes a universality result for linear regression with implicit regularization in the overparameterized regime, where the number of features $d$ exceeds the number of samples $n$.  The theorem demonstrates that the test error for the linear model trained on any feature matrix satisfying certain technical conditions is asymptotically equivalent to the test error of the same linear model trained on the Gaussian distribution with matching covariance. Gaussian Universality simplifies the analysis of model performance, making it tractable to predict the generalization error using techniques for working with Gaussian data, such as Gaussian Comparison Inequalities. In this subsection, for the sake of completeness, we present Theorem \ref{thm: univ}. The following assumptions are required for Theorem \ref{thm: univ}:

\begin{assumptions} \label{ass: univ}
\begin{enumerate}
    \item \textbf{Feature Matrix \(\bX \in \mathbb{R}^{n \times d}\)}: The rows of \(\bX\), denoted \(\mathbf{x}_i \in \mathbb{R}^d\) for \(i = 1, \dots, n\), are independently and identically distributed (i.i.d.) from a distribution \(\mathbb{P}\) with mean \(\bmu \in \mathbb{R}^d\) and covariance \(\bSigma \in \mathbb{R}^{d \times d}\). The distribution satisfies:
        \begin{itemize}
            \item Bounded moments up to the sixth order: For each row \(\mathbf{x}_i\), \(\mathbb{E}[\|\mathbf{x}_i - \bmu\|_2^q] = O(1)\) for \(q \leq 6\).
            \item Bounded mean: \(\|\bmu\|_2^2 = O(1)\).
            \item Covariance condition: For any fixed vector \(\mathbf{v} \in \mathbb{R}^d\), the quadratic form \(\mathbf{v}^T \bSigma \mathbf{v}\) has vanishing variance in the sense that \(\text{Var}(\mathbf{x}_i^T \mathbf{v}) = O(1/d)\) as \(d \to \infty\).
            \item Minimum singular value: The smallest singular value of \(\bX \bX^T \in \mathbb{R}^{n \times n}\), denoted \(\sigma_{\min}(\bX \bX^T)\), satisfies \(\sigma_{\min}(\bX \bX^T) = \Omega(1)\) with high probability, ensuring \(\bX\) is well-conditioned.
        \end{itemize}
    \item \textbf{Target Labels (\(\by \in \mathbb{R}^n\))}: The labels \(\by\) are generated as \(\by = \mathbf{X} \mathbf{w}^* + \boldsymbol{\epsilon}\), where \(\mathbf{w}^* \in \mathbb{R}^d\) is a fixed true parameter vector with \(\|\mathbf{w}^*\|_2 = O(1)\), and the noise \(\boldsymbol{\epsilon} \in \mathbb{R}^n\) has i.i.d. sub-Gaussian entries with mean zero and variance \(\sigma^2 = O(1)\).
    \item \textbf{Mirror Map (\(\psi: \mathbb{R}^d \to \mathbb{R}\))}: The mirror map \(\psi\) is \(M\)-strongly convex (i.e., \(\nabla^2 \psi \succeq M \bI_d\) for some \(M > 0\)), three times differentiable with bounded third derivatives (\(\|\nabla^3 \psi\| = O(1)\)), and satisfies \(\psi(\mathbf{0}) = O(d)\). Moreover, the gradient of the mirror map at the solution \(\bw_\bX\), denoted \(\nabla \psi(\bw_\bX)\), satisfies \(\|\nabla \psi(\bw_\bX)\|_2 = O(\sqrt{d})\) with high probability.
    \item \textbf{Overparameterization}: The dimensions \(d\) (number of parameters) and \(n\) (number of samples) tend to infinity with a fixed ratio \(d/n = \kappa > 1\), ensuring an overparameterized regime where the number of parameters exceeds the number of samples.
\end{enumerate}
\end{assumptions}
\begin{theorem}[\cite{ghane2024universality}]

\label{thm: univ}
Let $\bX \in \mathbb{R}^{n \times d}$ be a feature matrix whose rows are sampled from a disribution $\bbP$ with mean $\bmu \in \bbR^d$ and covariance $\bSigma \in \bbR^{d\times d}$ and $\by \in \mathbb{R}^n$ be the labels satisfying Assumptions \ref{ass: univ}. Let $\bG \in \mathbb{R}^{n \times d}$ be a matrix with independent rows sampled from $\mathcal{N}(\bmu,\bSigma)$. Define $\bw_\bX$ and $\bw_\bG$ to be the SMD solutions with a mirror $\psi$ trained on $\bX$ and $\bG$ respectively for some initialization $\bw_0$:
\begin{align}
& \bw_\bX = \arg\min_{\mathbf{w}} D_\psi(\mathbf{w}, \mathbf{w}_0) \text{ subject to } \mathbf{X}\mathbf{w} = \mathbf{y} \\
& \bw_\bG = \arg\min_{\mathbf{w}} D_\psi(\mathbf{w}, \mathbf{w}_0) \text{ subject to } \mathbf{G}\mathbf{w} = \mathbf{y} 
\end{align}

Then, asymptotically, the following holds for any Lipschitz function $g$ in probability:
\[
\lim_{n \rightarrow \infty} \Bigl|g(\bw_\bX) - g(\bw_\bG)\Bigr| = 0 
\]
In particular, taking $g(\bw) = \sqrt{\bw^T \bSigma \bw}$ ensures that $\bw_\bX$ and $\bw_\bG$ yield equal test MSE losses as $d$ grows large. 
\end{theorem}

\subsection{Gaussian Equivalence}

We utilize the \textbf{Gaussian Equivalence Principle (GEP)} to characterize the covariance matrices of the outputs of the Random Features  layers. Recall that the latter outputs are defined via 
$$\phi_{\ell}(\mathbf{\bx^{\ell-1}}) = \phi(\bW^{(\ell)}\bx^{(\ell - 1)}) \text{ for }\ell=1, \dots, L$$ 
Here, $\bx^{(0)} = \bx \in \mathbb{R}^d$ is the input, \(\mathbf{\bW^{(\ell)}} \in \mathbb{R}^{ d_\ell \times d_{\ell-1}}\) is a random weight matrix with i.i.d. entries. Namely, $$\mathbf{W}_{ij} \sim \mathcal{N}\left(0, \frac{1}{d_{\ell-1}}\right)$$ for the full precision model and $$\mathbf{W}_{ij} \sim \frac{1}{\sqrt{d_{\ell-1}}}\mathbf{Unif}(-1,+1)$$ for the one-bit quantized model, $d_0 = d$ and \(\phi: \mathbb{R} \to \mathbb{R}\) is an odd nonlinearity function. Define the covariance of the $\ell$-th hidden layer by $\bSigma_\ell$, i.e.
$$\bSigma_{\ell} = \bbE \bx^{(\ell)}{\bx^{(\ell)}}^T$$
In the proportional high-dimensional limit $$n, d, d_1,\dots,d_L \to \infty, \quad n/d = \Theta(1), \quad  \frac{d_{\ell-1}}{d_\ell} = \Theta(1), \quad \ell = 1,\dots L$$ GEP provides a recipe for finding $\bSigma_\ell$ via the following recursive relations:
$$\bSigma_{\ell} \approx \rho_{\ell,1}^2 \bW_{\ell} \bSigma_{\ell-1} \bW_{\ell}^\top  + \rho_{\ell,2}^2\bI_{d_\ell}$$
where
\begin{align*}
& \rho_{\ell,1} = \frac{1}{\sigma^2_{\ell-1}} \mathbb{E}_{z \sim \mathcal{N}(0, \sigma^2_{\ell-1})}z\phi(z) \\ 
& \rho_{\ell,2}^2 = \mathbb{E}_{z \sim \mathcal{N}(0, \sigma^2_{\ell-1})}\phi(z)^2 - \sigma^2_{\ell-1}\rho_{\ell,1}^2 \\
& \sigma_{\ell}^2 = \frac{\tr(\bSigma_\ell)}{d_\ell}
\end{align*}
with the initial conditions $$\bSigma_0 = \bSigma \text{ and } \sigma_0^2 = \frac{\tr(\bSigma)}{d}$$

\section{Related Works}\label{sec: related}
% In this section we provide a brief overview of existing works on Random Features model related to our setting and also Gaussian universality. 
% Random Features model \cite{rahimi2007random}, has been the subject of great study in recent years. The generalization error of the RF model with a single hidden layer has been analyzed in many different contexts in the high-dimensional proportional regime ; These include when the last layer is trained using a ridge regression objective\cite{gerace2020generalisation, dhifallah2020precise, ghorbani2021linearized, mei2022generalization, goldt2022gaussian}, or $\ba$ is taken to be the minimum $\ell_2$-norm interpolating solution \cite{hastie2022surprises}. Furthermore, for the binary classification task, the performance of the last layer as the $\ell_2$ \cite{montanari2019generalization} or $\ell_1$ \cite{liang2022precise} max-margin classifier has been analyzed. As Random Features model resembles neural networks at initialization, in the line of work \cite{moniri2023theory}, the generalization error after taking one-step of gradient descent on the hidden layer has been considered. Other settings include adversarial training \cite{hassani2024curse}, the attention mechanism as a Random Features model \cite{fu2023can}, and more recently, RFs have been explored in the non-asymptotic regime \cite{defilippis2024dimension}

In this section, we provide a brief overview of existing works relevant to our setting.

The Random Features (RF) model \cite{rahimi2007random} has been the subject of extensive study in recent years. The generalization error of the RF model with a single hidden layer has been analyzed in many different contexts within the high-dimensional proportional regime. These include settings where the last layer is trained using a ridge regression objective \cite{gerace2020generalisation, dhifallah2020precise, ghorbani2021linearized, mei2022generalization, goldt2022gaussian}, or where $\ba$ is taken to be the minimum $\ell_2$-norm interpolating solution \cite{hastie2022surprises}. Furthermore, for binary classification tasks, the performance of the last layer as either an $\ell_2$ \cite{montanari2019generalization} or $\ell_1$ \cite{liang2022precise} max-margin classifier has been analyzed. Since the Random Features model resembles neural networks at initialization, one line of work \cite{moniri2023theory} has considered the generalization error after taking a single step of gradient descent on the hidden layer. Other settings studied include adversarial training \cite{hassani2024curse}, the attention mechanism as a Random Features model \cite{fu2023can}, and more recently, RFs in the non-asymptotic regime \cite{defilippis2024dimension}.

RFs with multiple hidden layers have remained underexplored compared to those with single hidden layer. The paper \cite{schroder2023deterministic} rigorously proved the Gaussian universality of the test error for the last layer trained using ridge regression on the same task as described in Subsection \ref{subs: notation}. A concurrent paper \cite{bosch2023precise}  proved a similar universality result for much more general convex losses and regularizers. Furthermore, \cite{schroder2023deterministic}  provided a conjecture for the universality of the test error for more general convex losses and regularizers, as well as for cases where the structures of the learner and the ground-truth differ. In \cite{schroder2024asymptotics}, they extended these results to networks whose weights are not necessarily isotropic, imposing a general covariance structure on the weights per layer for ridge regression with squared loss. They went beyond the well-specified settings of \cite{bosch2023precise, schroder2023deterministic} and provided an expression for the test error where the ground truth and the learner features differ. They also conjectured a Gaussian equivalence model for multiple layers. To investigate the effect of the covariance structure of weights on the performance of RFs, \cite{zavatone2023learning} used the non-rigorous replica method to characterize the test error of a linear Random Features model, where the last layer is trained using ridge regression to learn a linear ground truth function. In \cite{cui2023bayes}, the authors computed the Bayes-optimal test error for estimating the target function in both classification and regression tasks for a deep Random Features model. They also provided a conjecture for the recursion on the population covariance of the layers, which was mentioned by \cite{bosch2023precise, schroder2023deterministic, schroder2024asymptotics}.

% The Gaussian Equivalence property, stated by \cite{goldt2020modeling} [Will add more to this in spirit of Section 2]

% Gaussian universality plays a key role in reducing a problem with non-gaussian distribution to one involving Gaussians matching the first and second moments of the original distribution. This phenomena has been investigated actively for various statistical inference problems, such as the universality of the test error of classifiers/regressors obtained through ridge regression or running gradient descent. For an incomplete list please see: \cite{montanari2017universality, panahi2017universal, oymak2018universality, abbasi2019universality, montanari2022universality, han2023universality, lahiry2023universality, dandi2023universality, ghane2024universality}. In the context of random Features model, the universality of the test error for regression was rigorously proved by \cite{hu2022universality}.

Gaussian universality plays a key role in reducing problems with non-Gaussian distributions to equivalent problems involving Gaussian distributions that match the first and second moments of the original distribution. This phenomenon has been actively investigated for various statistical inference problems, such as the universality of the test error of classifiers and regressors obtained through ridge regression or gradient descent. For an incomplete list, see \cite{montanari2017universality, panahi2017universal, oymak2018universality, abbasi2019universality, montanari2022universality, han2023universality, lahiry2023universality, dandi2023universality, ghane2024universality, ghane2025concentration}. In the context of Random Features models, the universality of the test error for regression was rigorously proved by \cite{hu2022universality, bosch2023precise, montanari2022universality, schroder2023deterministic}.

The Gaussian Equivalence property is a framework used in the context of Random Features models. It allows for recursive characterization of layer-wise statistics and provides theoretical justification for analyzing neural networks through their Gaussian approximations. An interested reader can refer to Section ~\ref{sec:preliminaries} for details on Gaussian Universality and Gaussian Equivalence. This principle has been used in many recent works, such as \cite{goldt2020modeling,bosch2023precise, hu2022universality, schroder2023deterministic, schroder2024asymptotics, defilippis2024dimension}. The paper \cite{hu2022universality} was the first to provide a rigorous proof of Gaussian Equivalence for Random Features Models with one hidden layer. The subsequent papers \cite{bosch2023precise, schroder2023deterministic} have proved different forms of Gaussian Equivalence for deep RF models. It should be mentioned that all works mentioned in this paragraph operate under the assumptions that the random features are Gaussian. 

Theoretical analyses of quantization and pruning are limited in the literature.The investigated topics include post-training quantization \cite{zhang2025provable}, training-aware quantization \cite{askarihemmat2024qgen}, analysis of generalization error of linear models for binary classification \cite{akhtiamov2024regularized}, multiclass classification \cite{ghane2025concentration} and pruning in the context of random features model \cite{chang2021provable}.

\section{Main Results}
\label{sec:results}

The following theorem, which is the main result of our work, provides a precise asymptotic characterization of the test loss for the quantized and non-quantized Random Features Models. Since we obtain the same expression for both, we conclude that quantizing the hidden layers to one bit naively does not lead to any degradation of performance for the Random Features Models as long as the model and the dataset are big enough and both models are trained via SMD using the same smooth mirror function.  

\begin{theorem}\label{thm: main}
    Let $f(\mathbf{x}, \ba, \bW^{1}, \dots, \bW^{L})$ be the Random Features Model defined in Subsection \ref{subs: notation}, where $$\bW^{1} \in \bbR^{d_1 \times d}, \dots, \bW^{L} \in \bbR^{d_L \times d_{L-1}}$$ are either full precision weights sampled i.i.d. from $$\mathbf{W}^{\ell}_{ij} \sim \mathcal{N}\left(0, \frac{1}{d_{\ell-1}}\right)$$ or one-bit quantized weights sampled i.i.d. from $$\mathbf{W}^{\ell}_{ij} \sim \frac{1}{\sqrt{d_{\ell-1}}}\mathbf{Unif}(-1,+1)$$ Assume that 
    \begin{itemize}
        \item The data $\bx \in \bbR^d$ satisfies $$\bbE \bx =0$$ along with the LCP property from Definition \ref{def: LCP} with $$\sigma^2 = O \left(\frac{1}{d}\right)$$
        \item The activation function $\phi$ is odd and has bounded first, third and fifth derivatives.
        \item The dimension of the last layer $d_L$ exceeds the number of training samples $n$ and the last layer $\ba$ is trained to minimize the following objective using SMD with a mirror $\psi$ satisfying Assumptions \ref{ass: univ} initialized at $\ba_0 \in \bbR^{d_L}$: $$\min_{\ba} \sum_{i = 1}^n  \mathcal{L}\left(y_i - f(\mathbf{x}_i, \ba, \bW^{1}, \dots, \bW^{L})\right)$$
        \item The labels $y$ are generated using a ground truth $$\ba_* \sim \mathcal{N}(0, \frac{\bI}{d_L})$$
        as defined in \eqref{eq: labels}.
    \end{itemize}
    Then, in the asymptotic proportional regime $$n, d, d_1,\dots,d_L \to \infty,$$ 
    $$\frac{n}{d}, \frac{d}{d_1}, \dots, \frac{d_{L-1}}{d_L} = \Theta(1),$$ the test loss satisfies
    $$\mathbb{E}_{\bx}[(f(\mathbf{x}, \ba, \bW^{1}, \dots, \bW^{L})- y)^2] \to \tau = \tau^{(L)}$$
    Here, convergence means convergence in probability and $\tau$ can be found by solving a system of elaborate nonlinear scalar deterministic equations, which follow from \eqref{eq: mirr_final} for the case of general mirrors and are simplified for the case of SGD in \eqref{eq: sgd_final}. It should be noted that \eqref{eq: mirr_final} and \eqref{eq: sgd_final}  are min-max optimization objectives and $\tau$ can be found by solving the corresponding saddle-point equations.
    
    In particular, asymptotically, the error does not depend on the realizations of $$\bW^{1}, \dots, \bW^{L}$$ and does not change if we replace $$\bW^{1},\dots, \bW^{L} $$ by $$\frac{\sign(\bW^{1})}{\sqrt{d_1}}, \dots, \frac{\sign(\bW^{L})}{\sqrt{d_L}},$$
    where $\sign$ is applied entry-wise. 
\end{theorem}

\begin{remark}
    Examples of data satisfying LCP with $\sigma^2 = O\left(\frac{1}{d}\right)$ include $\bx = \bg \sim \mathcal{N}(0,\bSigma)$ for $\bSigma$ such that $$\tr(\bSigma) = O(1)$$ and $$\kappa(\bSigma) = \frac{\sigma_{\max}}{\sigma_{\min}} = O(1)$$ as well as $\bx = f(\bg)$ for any Lipschitz $f$ with bounded Lipschitz constant and the same $\bg$ defined as above. 
\end{remark}

\begin{remark}
    We observe a close match between the performances of Gaussian and Rademacher Random Features trained to classify points from MNIST dataset with ReLU activation function in Section \ref{sec:numerics}. As such, we believe that it should be possible to extend Theorem \ref{thm: main} to non-centered data and non-odd activation functions. The main technical obstacle for this  is establishing Gaussian Equivalence results applicable to the latter scenario. We leave this as an important direction for future work.  
\end{remark}

\begin{remark}
    While we postpone presentation of the exact non-linear equations from Theorem \ref{thm: main} defining $\tau$ to \eqref{eq: mirr_final} and \eqref{eq: sgd_final} in the Appendix C, we would like to provide the essence here.  
    To find $\tau$ for a general smooth mirror $\psi$, one needs to solve a nonlinear scalar deterministic system of equations involving $2L$ scalar parameters.  For the case of SGD, i.e. when the mirror $$\psi(\cdot)= \frac{1}{2}\|\cdot\|_2^2$$ the number of unknown parameters could be reduced to $L$. This way, we obtain a deterministic system of equations that defines the test MSE loss implicitly for both Gaussian and approprietly normalized Rademacher weights. To the best of our knowledge, our work is the first work characterizing the test loss for normalized Rademacher Random Features via a finite number of scalar equations.
    
    % Furthermore, $\tau^{(L)}$ could be found by first solving the aforementioned deterministic scalar system of nonlinear equations in $L$ scalar variables $\zeta_1, \cdots, \zeta_{L-1}$ and $\theta$ and then a system of linear equations in L variables. For instance, to demonstrate, for $\zeta_1$ we obtain a nonlinear equation of the form
    % \begin{align*}
    %   & \Bigl(\frac{\zeta_3}{8} + \frac{c_L^2 \rho^2_{1,2}}{\rho^2_{1,1}} \Bigr) \zeta_1 = \frac{1}{2} - \\
    %   & \zeta_1  \frac{\zeta^2_2 c^2_L}{2} \tr \Bigl( \zeta_1 c^2_{L}\zeta^2_2 \|\bg_1\|_2^2 \bSigma_0 + \zeta_2 \bI  \Bigr)^{-1} 
    % \end{align*}
    % Here $c_L := C \theta \prod_{i=3}^{L-1} \zeta_i$. For 
\end{remark}

\section{Our approach and contributions}\label{sec: approach}

Our approach is different from \cite{bosch2023precise} and \cite{schroder2023deterministic}, as we start with invoking Gaussian Universality for the last layer and only afterwards do we apply Gaussian Equivalence Principle to calculate the covariance of the last layer. This approach allows us to analyze the generalization error of the solutions obtained via Stochastic Mirror Descent with smooth mirrors and arbitrary convex losses, extending results available in the literature. Indeed, to the best of our knowledge, the only examples considered in the literature previously are ridge regression \cite{gerace2020generalisation, dhifallah2020precise, ghorbani2021linearized, mei2022generalization, goldt2022gaussian}, SGD initialized at $0$ \cite{hastie2022surprises} and the $\ell_2$ \cite{montanari2019generalization} and $\ell_1$ \cite{liang2022precise} max-margin classifiers. In addition, we prove Gaussian Equivalence for deep Random Features Models ($L > 1$), while \cite{schroder2023deterministic} leaves the case $L>1$ as a conjecture for objectives other than ridge regression and \cite{bosch2023precise} takes an additional expectation with respect to the weights in the Gaussian Equivalence part.

Other works \cite{montanari2022universality, defilippis2024dimension, schroder2024asymptotics} are  more similar to the present paper, as they apply similar universality results to the output of the last layer as well. The main differences between \cite{montanari2022universality, defilippis2024dimension, schroder2024asymptotics}  and our work is that we extend their results to the case of normalized Rademacher features to capture the one-bit quantization of weights as well as apply additional steps to show that the test error converges to a deterministic quantity independent of the realizations of the weights.

After combining Gaussian Universality with Gaussian Equivalence, we proceed to apply Convex Gaussian Min-Max Theorem (CGMT) \cite{thrampoulidis2014tight, akhtiamov2024novel} to each hidden layer one by one to prove that the error concentrates with respect to the randomness in each $\bW^{\ell}$ as well. For Gaussian weights, this application is more straightforward, while for normalized Rademacher weights we have to employ an additional step and apply another result \cite{han2023universality} that says that CGMT can be applied to many other i.i.d. subgaussian designs. This allows us to derive identical expressions for the test losses for both Gaussian and Rademacher models and conclude that one-bit quantization does not lead to any deterioration in performance for Random Features Models. This aspect of our work is novel as well: to the best of our knowledge, our work is the first to derive expressions for deep non-Gaussian Random Features.

% Finally, to the best of the authors' knowledge, all other works in the literature consider the case of ridge regression or other objectives with an explicit regularization term. In our setting, we assume that the model is trained using Stochastic Mirror Descent for a smooth mirror and a convex loss. 

\section{Numerical Experiments}

\label{sec:numerics}

We validate our theoretical results through experiments on Random Features models with Gaussian Weights and with one-bit quantized weights with last layer trained on synthetic Gaussian data and MNIST\cite{deng2012mnist}. For the Gaussian data, we used tanh activation function and trained the last layer with SGD as well as with negative entropy mirror. For MNIST, we use ReLU activations and trained the last layer using SGD.

\subsection{One-Bit Quantization}

\subsubsection{Synthetic data}

We verify that one-bit quantization incurs no loss by comparing test MSE between Gaussian and Rademacher weights across depths (defined as the numebr of hidden layers) $$L \in \{1, 2, 3, 4, 5\}$$ Specifically, we compare two Random Features variants, Gaussian $$\mathbf{W}_{ij}^{(\ell)} \sim \mathcal{N}\left(0, \frac{1}{d_{\ell-1}}\right)$$ and Rademacher $$\mathbf{W}_{ij}^{(\ell)} \sim \frac{1}{\sqrt{d_{\ell-1}}}\text{Unif}\{-1, +1\}$$
We generate synthetic data with $\mathbf{x}_i \sim \mathcal{N}(0, \frac{\mathbf{I}_d}{d})$ and labels $y_i = \phi_L(\mathbf{x}_i)^\top \mathbf{a}_*$, where $\phi_L$ is the $L$-layer random features map with $\tanh$ activation, $\mathbf{a}_* \sim \mathcal{N}(0, \frac{1}{d_L}\mathbf{I})$. This is the data-generation procedure that will be used for demonstrating inference speedup. We use $n = 1000$ training samples,  input dimension $d = 8192$ and hidden dimensions $d_1 = \dots = d_L = 4096$ for each hidden layer. Following the overparametrized regime, we fix random features and train only the last layer via minimum $\ell_2$ norm  solution, which can be recovered analytically as: $$\mathbf{a} = \Phi^\top (\Phi \Phi^\top)^{-1} \mathbf{y}\text{, where } \Phi \in \mathbb{R}^{n \times d_L}$$ As can be seen in Figure \ref{fig:synt}, we observe a close match between the test error of the RF model with Gaussian weights and the RF model with Rademacher weights. To illustrate a more general case of our theorem, we also consider the negative Shannon entropy $$\psi(\bw) = \sum |w_i|\log(|w_i|)$$ under the same setting in Figure \ref{fig:synt_entr}. For both scenarios, we use $n_{test} = 5000$ test samples for estimating the test MSE loss.

\begin{figure}[htb]
    \centering
    \includegraphics[width=0.6\linewidth]{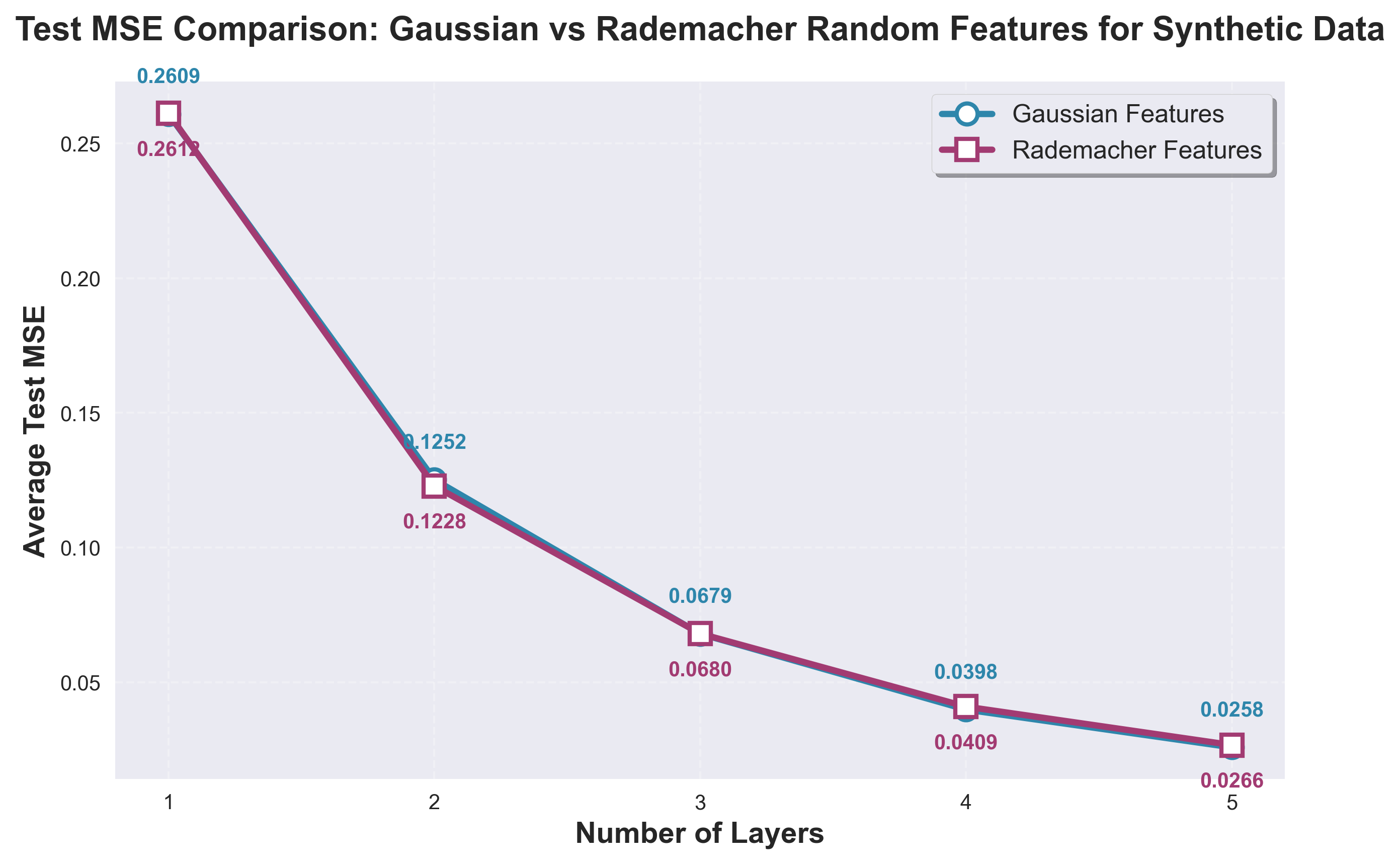}
    \caption{Random Features with varying depth for Synthetic Data for SGD}
    \label{fig:synt}
\end{figure}

\begin{figure}[htb]
    \centering
    \includegraphics[width=0.6\linewidth]{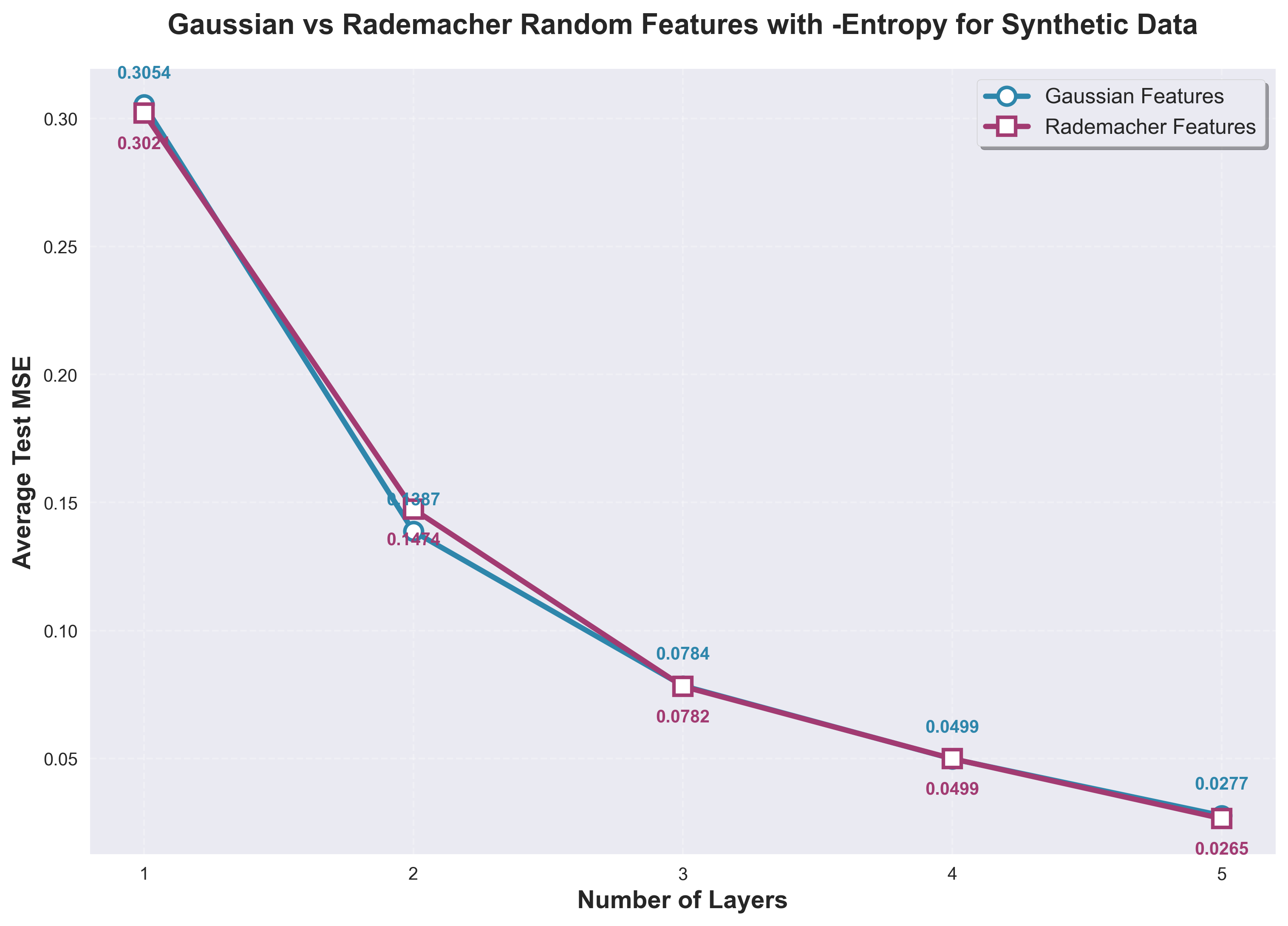}
    \caption{Random Features with varying depth for Synthetic Data for Negative Entropy Mirror}
    \label{fig:synt_entr}
\end{figure}

\subsubsection{MNIST}

We run the following two experiments for MNIST:

\begin{itemize}
\item We train random feature networks with Gaussian and Rademacher weights and with ReLU activations on MNIST, 
varying the number of layers $$L \in \{1,2,3,4,5\}$$ while fixing the hidden dimensions  $$d_1 = d_2 = \dots d_L = 512$$ Since this is a classification task, we report test accuracy rather than test MSE. For each layer count, we use $20$ samples per class and average results 
over $20$ trials. We use one-hot encoding of the classes in the optimization objective. The final layer uses minimum $\ell_2$-norm interpolation. The results are presented in Figure \ref{fig:mnist_vs_width} and demonstrate a close match, despite not being covered by our theory. We use a total of $200$ test samples for estimating the resulting test accuracy. 

\item We also train $L = 2$ - layer random feature networks with Gaussian and Rademacher weights 
on MNIST, varying the hidden dimension $$d_1 = d_2 \in \{256,512,1024,2048,4196\}$$ while 
fixing $L = 2$. Since this is a classification task, we report test accuracy rather than test MSE. For each width, we perform evaluation in the same way as in the experiment from the previous bullet point.

\end{itemize}

\begin{figure}[htb]
    \centering
    \includegraphics[width=0.6\linewidth]{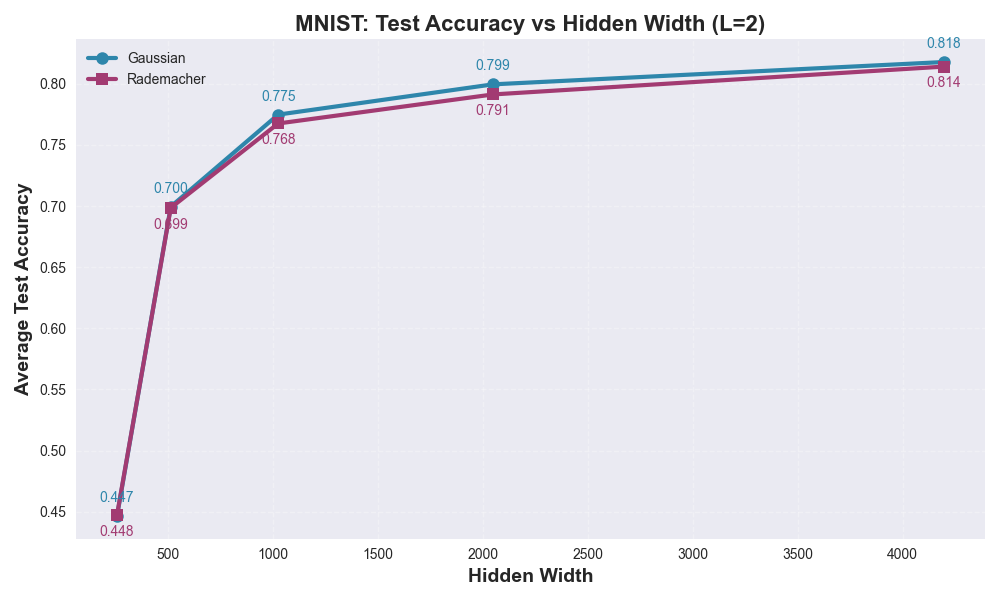}
    \caption{Random Features with varying depth for MNIST}
    \label{fig:mnist_vs_width}
\end{figure}

\begin{figure}[htb]
    \centering
    \includegraphics[width=0.6\linewidth]{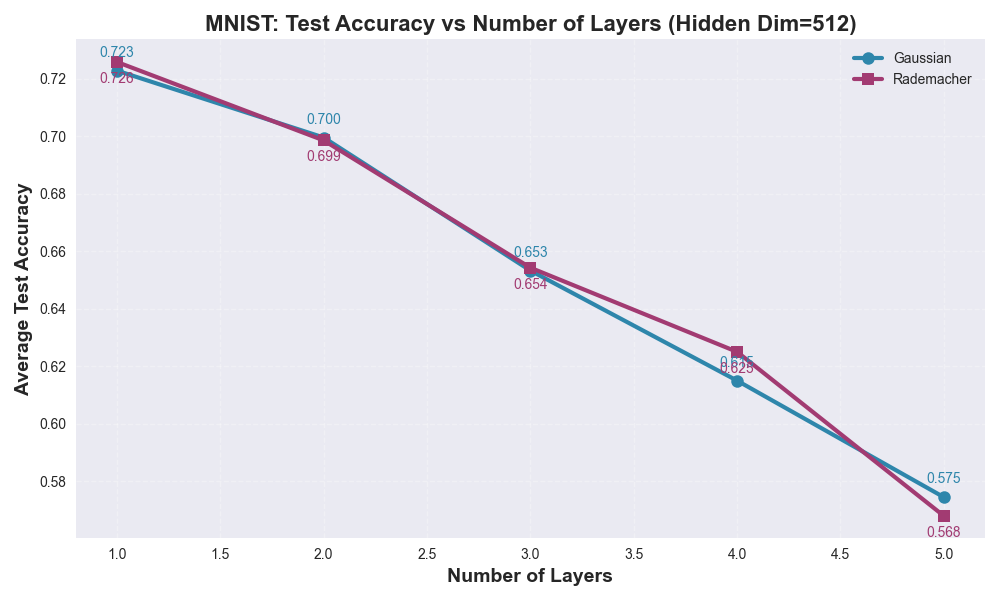}
    \caption{Random Features with varying hidden width for MNIST}
    \label{fig:mnist_vs_depth}
\end{figure}

% \subsection{Inference Speedup}

% We benchmark inference comparing standard PyTorch (FP32) against optimized 1-bit kernels (GemLite library):
% \begin{itemize}
%     \item Test configuration: $d = 16384$, hidden dimensions from $0.25d$ to $2d$
%     \item Protocol: 
%     \item Metrics: inference time (ms), speedup ratio, memory usage
% \end{itemize}

% As Figure \ref{fig:gemlite} demonstrates, we obtain a $\sim 4X$ inference speed up from reduced memory bandwidth and optimized binary operations, $32 \times$ memory reduction enabling edge deployment.

\subsection{Inference speedup}

For investigating the potential speedup of employing one-bit weights during inference, we consider the setting in Section 6.1.1 for  $n = 1000$  training samples. We proceed to load the model using PyTorch with CUDA acceleration, on an RTX 2060 laptop GPU with 6GB VRAM in FP32 precision. Furthermore, we leverage the Gemlite \cite{gemlite2024}, a triton-based kernel library with 1 bit weights and group size set to 64 with 500 warmup runs and 50,000 timed iterations on batch size 1. We present the results in Figure \ref{fig:gemlite} for the Random Features model with one hidden layer. We observe a 4 times speed-up on average.

\begin{figure}[htb]
    \centering
    \includegraphics[width=0.6\linewidth]{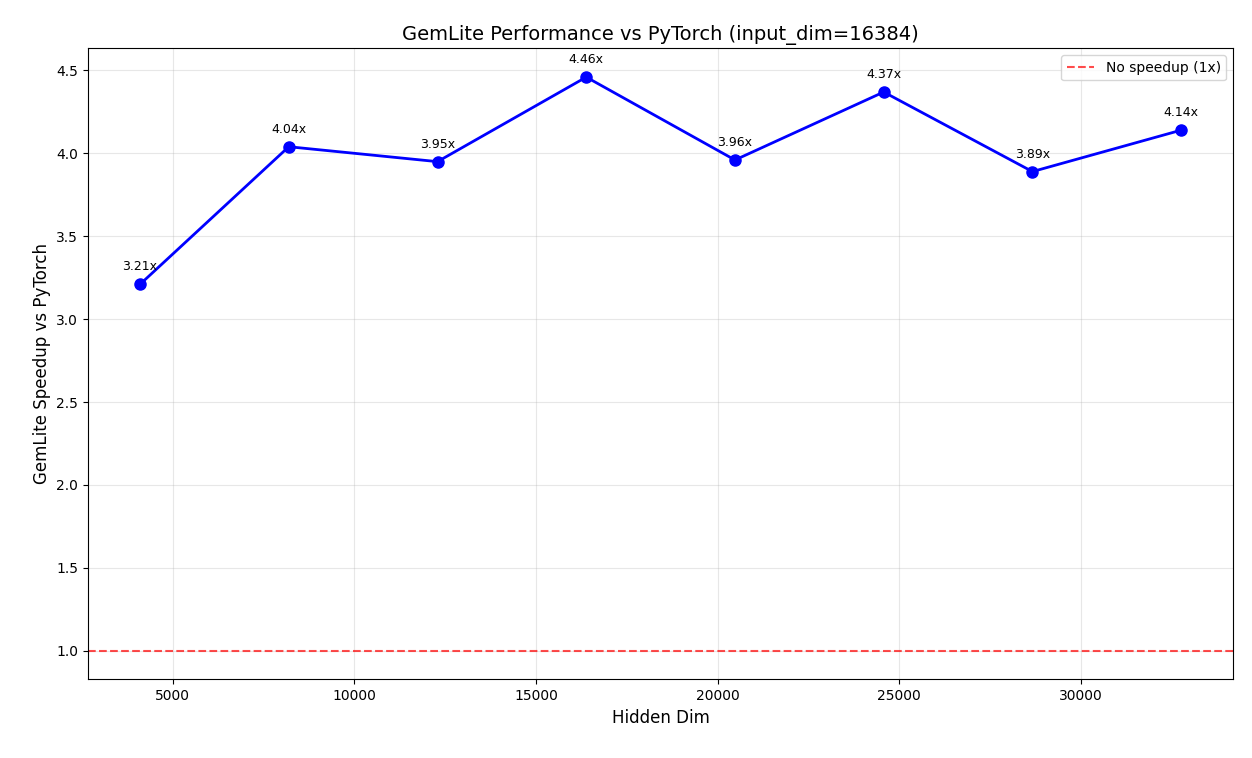}
    \caption{Inference speed up vs Hidden Dimension}
    \label{fig:gemlite}
\end{figure}

\section{Conclusion and Future Works}

The present paper leverages a combination of Gaussian Universality and Gaussian Equivalence principles followed by an application of Gaussian Comparison Inequalities to analyze one-bit compression of the weights for Random Features Models. We demonstrate that for random features the naive one-bit compression is lossless and results in in a $\sim 4X$ inference speed up assuming the hidden layer dimension  is sufficiently wide. It is worth mentioning that quantizing the last layer under the same setting would neither be lossless nor result in a noticeable inference speed up. As such, we suggest that the last layer should never be quantized in practice. 

Our experiments suggest that one-bit quantization might be lossless for Random Features with ReLU trained on classification tasks as well. This calls for extending our methods to classification instead of regression and to non-odd activation functions. Both former and latter extensions would require a Gaussian Equivalence Principle for non-centered data. 

Finally, while we believe that the setting of random features considered in the present work sheds light on one-bit quantization, it would be interesting to study the more nuanced picture of learnable representations. While performing the full analysis might be too challenging in general, we suggest starting with the simpler case when the features are learned via one-step Gradient Descent \cite{moniri2023theory}.  In the latter setting, it would be interesting to see the effects of more sophisticated one-bit compression techniques as well.

\bibliographystyle{apalike}
\bibliography{main}

\begin{thebibliography}{}

\bibitem[Abbasi et~al., 2019]{abbasi2019universality}
Abbasi, E., Salehi, F., and Hassibi, B. (2019).
\newblock Universality in learning from linear measurements.
\newblock {\em Advances in Neural Information Processing Systems}, 32.

\bibitem[Akhtiamov et~al., 2024a]{akhtiamov2024novel}
Akhtiamov, D., Bosch, D., Ghane, R., Varma, K.~N., and Hassibi, B. (2024a).
\newblock A novel gaussian min-max theorem and its applications.
\newblock {\em arXiv preprint arXiv:2402.07356}.

\bibitem[Akhtiamov et~al., 2024b]{akhtiamov2024regularized}
Akhtiamov, D., Ghane, R., and Hassibi, B. (2024b).
\newblock Regularized linear regression for binary classification.
\newblock In {\em 2024 IEEE International Symposium on Information Theory (ISIT)}, pages 202--207. IEEE.

\bibitem[AskariHemmat et~al., 2024]{askarihemmat2024qgen}
AskariHemmat, M., Jeddi, A., Hemmat, R.~A., Lazarevich, I., Hoffman, A., Sah, S., Saboori, E., Savaria, Y., and David, J.-P. (2024).
\newblock Qgen: On the ability to generalize in quantization aware training.
\newblock {\em arXiv preprint arXiv:2404.11769}.

\bibitem[Azizan et~al., 2021]{azizan2021stochastic}
Azizan, N., Lale, S., and Hassibi, B. (2021).
\newblock Stochastic mirror descent on overparameterized nonlinear models.
\newblock {\em IEEE Transactions on Neural Networks and Learning Systems}, 33(12):7717--7727.

\bibitem[Badri and {Mobius Labs}, 2024]{gemlite2024}
Badri, H. and {Mobius Labs} (2024).
\newblock Gemlite: Fast low-bit matmul kernels in triton.
\newblock GitHub repository.

\bibitem[Bosch et~al., 2023]{bosch2023precise}
Bosch, D., Panahi, A., and Hassibi, B. (2023).
\newblock Precise asymptotic analysis of deep random feature models.
\newblock In {\em The Thirty Sixth Annual Conference on Learning Theory}, pages 4132--4179. PMLR.

\bibitem[Chang et~al., 2021]{chang2021provable}
Chang, X., Li, Y., Oymak, S., and Thrampoulidis, C. (2021).
\newblock Provable benefits of overparameterization in model compression: From double descent to pruning neural networks.
\newblock In {\em Proceedings of the AAAI Conference on Artificial Intelligence}, volume~35, pages 6974--6983.

\bibitem[Chee et~al., 2023]{chee2023quip}
Chee, J., Cai, Y., Kuleshov, V., and De~Sa, C.~M. (2023).
\newblock Quip: 2-bit quantization of large language models with guarantees.
\newblock {\em Advances in Neural Information Processing Systems}, 36:4396--4429.

\bibitem[Cui et~al., 2023]{cui2023bayes}
Cui, H., Krzakala, F., and Zdeborov{\'a}, L. (2023).
\newblock Bayes-optimal learning of deep random networks of extensive-width.
\newblock In {\em International Conference on Machine Learning}, pages 6468--6521. PMLR.

\bibitem[Dandi et~al., 2023]{dandi2023universality}
Dandi, Y., Stephan, L., Krzakala, F., Loureiro, B., and Zdeborov{\'a}, L. (2023).
\newblock Universality laws for gaussian mixtures in generalized linear models.
\newblock {\em Advances in Neural Information Processing Systems}, 36:54754--54768.

\bibitem[Defilippis et~al., 2024]{defilippis2024dimension}
Defilippis, L., Loureiro, B., and Misiakiewicz, T. (2024).
\newblock Dimension-free deterministic equivalents for random feature regression.
\newblock {\em arXiv preprint arXiv:2405.15699}.

\bibitem[Deng, 2012]{deng2012mnist}
Deng, L. (2012).
\newblock The mnist database of handwritten digit images for machine learning research.
\newblock {\em IEEE Signal Processing Magazine}, 29(6):141--142.

\bibitem[Dhifallah and Lu, 2020]{dhifallah2020precise}
Dhifallah, O. and Lu, Y.~M. (2020).
\newblock A precise performance analysis of learning with random features.
\newblock {\em arXiv preprint arXiv:2008.11904}.

\bibitem[Frantar et~al., 2022]{frantar2022gptq}
Frantar, E., Ashkboos, S., Hoefler, T., and Alistarh, D. (2022).
\newblock Gptq: Accurate post-training quantization for generative pre-trained transformers.
\newblock {\em arXiv preprint arXiv:2210.17323}.

\bibitem[Fu et~al., 2023]{fu2023can}
Fu, H., Guo, T., Bai, Y., and Mei, S. (2023).
\newblock What can a single attention layer learn? a study through the random features lens.
\newblock {\em Advances in Neural Information Processing Systems}, 36:11912--11951.

\bibitem[Gerace et~al., 2020]{gerace2020generalisation}
Gerace, F., Loureiro, B., Krzakala, F., M{\'e}zard, M., and Zdeborov{\'a}, L. (2020).
\newblock Generalisation error in learning with random features and the hidden manifold model.
\newblock In {\em International Conference on Machine Learning}, pages 3452--3462. PMLR.

\bibitem[Ghane et~al., 2024]{ghane2024universality}
Ghane, R., Akhtiamov, D., and Hassibi, B. (2024).
\newblock Universality in transfer learning for linear models.
\newblock {\em arXiv preprint arXiv:2410.02164}.

\bibitem[Ghane et~al., 2025]{ghane2025concentration}
Ghane, R., Bao, A., Akhtiamov, D., and Hassibi, B. (2025).
\newblock Concentration of measure for distributions generated via diffusion models.
\newblock {\em arXiv preprint arXiv:2501.07741}.

\bibitem[Ghorbani et~al., 2021]{ghorbani2021linearized}
Ghorbani, B., Mei, S., Theodor, M., and Montanari, A. (2021).
\newblock Linearized two-layers neural networks in high dimension.
\newblock {\em The Annals of Statistics}, 49(2):1029--1054.

\bibitem[Goldt et~al., 2022]{goldt2022gaussian}
Goldt, S., Loureiro, B., Reeves, G., Krzakala, F., M{\'e}zard, M., and Zdeborov{\'a}, L. (2022).
\newblock The gaussian equivalence of generative models for learning with shallow neural networks.
\newblock In {\em Mathematical and Scientific Machine Learning}, pages 426--471. PMLR.

\bibitem[Goldt et~al., 2020]{goldt2020modeling}
Goldt, S., M{\'e}zard, M., Krzakala, F., and Zdeborov{\'a}, L. (2020).
\newblock Modeling the influence of data structure on learning in neural networks: The hidden manifold model.
\newblock {\em Physical Review X}, 10(4):041044.

\bibitem[Han and Shen, 2023]{han2023universality}
Han, Q. and Shen, Y. (2023).
\newblock Universality of regularized regression estimators in high dimensions.
\newblock {\em The Annals of Statistics}, 51(4):1799--1823.

\bibitem[Hassani and Javanmard, 2024]{hassani2024curse}
Hassani, H. and Javanmard, A. (2024).
\newblock The curse of overparametrization in adversarial training: Precise analysis of robust generalization for random features regression.
\newblock {\em The Annals of Statistics}, 52(2):441--465.

\bibitem[Hastie et~al., 2022]{hastie2022surprises}
Hastie, T., Montanari, A., Rosset, S., and Tibshirani, R.~J. (2022).
\newblock Surprises in high-dimensional ridgeless least squares interpolation.
\newblock {\em Annals of statistics}, 50(2):949.

\bibitem[Hu and Lu, 2022]{hu2022universality}
Hu, H. and Lu, Y.~M. (2022).
\newblock Universality laws for high-dimensional learning with random features.
\newblock {\em IEEE Transactions on Information Theory}, 69(3):1932--1964.

\bibitem[Jacot et~al., 2018]{jacot2018neural}
Jacot, A., Gabriel, F., and Hongler, C. (2018).
\newblock Neural tangent kernel: Convergence and generalization in neural networks.
\newblock {\em Advances in neural information processing systems}, 31.

\bibitem[Lahiry and Sur, 2023]{lahiry2023universality}
Lahiry, S. and Sur, P. (2023).
\newblock Universality in block dependent linear models with applications to nonparametric regression.
\newblock {\em arXiv preprint arXiv:2401.00344}.

\bibitem[Liang and Sur, 2022]{liang2022precise}
Liang, T. and Sur, P. (2022).
\newblock A precise high-dimensional asymptotic theory for boosting and minimum-$\ell_1$ 1-norm interpolated classifiers.
\newblock {\em The Annals of Statistics}, 50(3):1669--1695.

\bibitem[Ma et~al., 2024]{ma2024era}
Ma, S., Wang, H., Ma, L., Wang, L., Wang, W., Huang, S., Dong, L., Wang, R., Xue, J., and Wei, F. (2024).
\newblock The era of 1-bit llms: All large language models are in 1.58 bits.
\newblock {\em arXiv preprint arXiv:2402.17764}, 1(4).

\bibitem[Mei and Montanari, 2022]{mei2022generalization}
Mei, S. and Montanari, A. (2022).
\newblock The generalization error of random features regression: Precise asymptotics and the double descent curve.
\newblock {\em Communications on Pure and Applied Mathematics}, 75(4):667--766.

\bibitem[Moniri et~al., 2023]{moniri2023theory}
Moniri, B., Lee, D., Hassani, H., and Dobriban, E. (2023).
\newblock A theory of non-linear feature learning with one gradient step in two-layer neural networks.
\newblock {\em arXiv preprint arXiv:2310.07891}.

\bibitem[Montanari and Nguyen, 2017]{montanari2017universality}
Montanari, A. and Nguyen, P.-M. (2017).
\newblock Universality of the elastic net error.
\newblock In {\em 2017 IEEE International Symposium on Information Theory (ISIT)}, pages 2338--2342. IEEE.

\bibitem[Montanari et~al., 2019]{montanari2019generalization}
Montanari, A., Ruan, F., Sohn, Y., and Yan, J. (2019).
\newblock The generalization error of max-margin linear classifiers: Benign overfitting and high dimensional asymptotics in the overparametrized regime.
\newblock {\em arXiv preprint arXiv:1911.01544}.

\bibitem[Montanari and Saeed, 2022]{montanari2022universality}
Montanari, A. and Saeed, B.~N. (2022).
\newblock Universality of empirical risk minimization.
\newblock In {\em Conference on Learning Theory}, pages 4310--4312. PMLR.

\bibitem[Oymak and Tropp, 2018]{oymak2018universality}
Oymak, S. and Tropp, J.~A. (2018).
\newblock Universality laws for randomized dimension reduction, with applications.
\newblock {\em Information and Inference: A Journal of the IMA}, 7(3):337--446.

\bibitem[Panahi and Hassibi, 2017]{panahi2017universal}
Panahi, A. and Hassibi, B. (2017).
\newblock A universal analysis of large-scale regularized least squares solutions.
\newblock {\em Advances in Neural Information Processing Systems}, 30.

\bibitem[Rahimi and Recht, 2007]{rahimi2007random}
Rahimi, A. and Recht, B. (2007).
\newblock Random features for large-scale kernel machines.
\newblock {\em Advances in neural information processing systems}, 20.

\bibitem[Schr{\"o}der et~al., 2023]{schroder2023deterministic}
Schr{\"o}der, D., Cui, H., Dmitriev, D., and Loureiro, B. (2023).
\newblock Deterministic equivalent and error universality of deep random features learning.
\newblock In {\em International Conference on Machine Learning}, pages 30285--30320. PMLR.

\bibitem[Schr{\"o}der et~al., 2024]{schroder2024asymptotics}
Schr{\"o}der, D., Dmitriev, D., Cui, H., and Loureiro, B. (2024).
\newblock Asymptotics of learning with deep structured (random) features.
\newblock {\em arXiv preprint arXiv:2402.13999}.

\bibitem[Seddik et~al., 2020]{pmlr-v119-seddik20a}
Seddik, M. E.~A., Louart, C., Tamaazousti, M., and Couillet, R. (2020).
\newblock Random matrix theory proves that deep learning representations of {GAN}-data behave as {G}aussian mixtures.
\newblock In III, H.~D. and Singh, A., editors, {\em Proceedings of the 37th International Conference on Machine Learning}, volume 119 of {\em Proceedings of Machine Learning Research}, pages 8573--8582. PMLR.

\bibitem[Thrampoulidis et~al., 2014]{thrampoulidis2014tight}
Thrampoulidis, C., Oymak, S., and Hassibi, B. (2014).
\newblock A tight version of the gaussian min-max theorem in the presence of convexity.
\newblock {\em arXiv preprint arXiv:1408.4837}.

\bibitem[Wang et~al., 2023]{wang2023bitnet}
Wang, H., Ma, S., Dong, L., Huang, S., Wang, H., Ma, L., Yang, F., Wang, R., Wu, Y., and Wei, F. (2023).
\newblock Bitnet: Scaling 1-bit transformers for large language models.
\newblock {\em arXiv preprint arXiv:2310.11453}.

\bibitem[Xu et~al., 2024]{xu2024survey}
Xu, M., Yin, W., Cai, D., Yi, R., Xu, D., Wang, Q., Wu, B., Zhao, Y., Yang, C., Wang, S., et~al. (2024).
\newblock A survey of resource-efficient llm and multimodal foundation models.
\newblock {\em arXiv preprint arXiv:2401.08092}.

\bibitem[Zavatone-Veth and Pehlevan, 2023]{zavatone2023learning}
Zavatone-Veth, J. and Pehlevan, C. (2023).
\newblock Learning curves for deep structured gaussian feature models.
\newblock {\em Advances in neural information processing systems}, 36:42866--42897.

\bibitem[Zhang et~al., 2025]{zhang2025provable}
Zhang, H., Zhang, S., Colbert, I., and Saab, R. (2025).
\newblock Provable post-training quantization: Theoretical analysis of optq and qronos.
\newblock {\em arXiv preprint arXiv:2508.04853}.

\end{thebibliography}

\appendix

\section{Scheme of the proof of Theorem \ref{thm: main}}

As outlined in Section \ref{sec: approach}, our approach is comprised of a consecutive application of Gaussian Universality, Gaussian Equivalence and Gaussian Comparison Inequalities. We present the omitted proofs related to Gaussian Universality and Gaussian Equivalence in Subsections \ref{sec: app_gaus_univ} and \ref{sec: app_gaus_equiv} respectively, followed by the missing CGMT derivations in Section \ref{sec: CGMT}. 

\section{Gaussian Universality and Gaussian Equivalence}

Denote the rows of $\bW^{(\ell)}$ by $\bw^{(\ell)}_1,\dots, \bw^{(\ell)}_{d_\ell}$

 Note that the following event holds w.h.p. with respect to randomness in $\bW^{(1)}, \dots, \bW^{(L)}$ 

\begin{align}\label{eq: generic_weights}
     \max_{0\leq i, j \leq d_\ell}\left|{\bw^{(\ell)}_i}^T\bw^{(\ell)}_j - \delta_{i, j} \right| \leq \frac{C}{\sqrt{d_\ell}} \text{ and }\|\bW^{(\ell)}\|_{\text{op}} = O(1)
\end{align}

Note that \eqref{eq: generic_weights} holds w.h.p. both when the weights are normalized i.i.d. Rademacher as well as standard Gaussian. For the purposes of this section, we freeze a realization of the features $\bW^{(\ell)}$ satisfying \eqref{eq: generic_weights} for $\ell = 1,\dots, L$ and consider the randomness w.r.t. the inputs $\bx$ only.

\subsection{Gaussian Equivalence}\label{sec: app_gaus_equiv}

\subsubsection{One hidden layer}

We will illustrate the argument for Random Features Models with one hidden layer first.  We would like to apply Theorem \ref{thm: univ}  to $\ba = \phi(\bW\bx)$. Denote $m = d_1$ the width of the only hidden layer in this case for ease of notation. Note that $\bbE_\bx \phi(\bW\bx)$ = 0 since $\phi$ is odd and denote $$\bSigma_{\ba} = \bbE_{\bx}  \phi(\bW\bx) \phi(\bW\bx)^T$$ 

Theorem \ref{thm: univ} guarantees that, even though $\ba$ is not Gaussian, the test error remains unchanged if we train the last layer on data sampled from $\ba' \sim \mathcal{N}(0, \bSigma_\ba)$ instead of $\ba$.

Thus, according to Lemma 5 from \cite{hu2022universality}, we have 
\begin{align}
    \|\bSigma_\ba - \bSigma_\bb\|_{op} = O\left(\|\bW\|_{op}\frac{\polylog(m)}{m^2}\right)
\end{align}
Here, $m$ is the width of the hidden layer and  $\bSigma_\bb$ is the covariance of the distribution defined via 
\begin{align*}
    & \bb \sim \rho_1\bW\bx + \rho_2\bg \\
    & \gamma \sim \mathcal{N}(0,1) \\
    & \rho_1 =  \bbE_\gamma \gamma\phi(\gamma) \\
    & \rho_2 =  (\bbE_\gamma \phi^2(\gamma) - \rho_1^2)^{\frac{1}{2}}
\end{align*}

Note that $\|\bW\|_{op} = O(1)$ holds w.h.p. as well. Therefore,
\begin{align}
    \|\bSigma_\ba - \bSigma_\bb\|_{op}  = o(\frac{1}{\sigma_{\min}(\bSigma_\ba)}) \text{ as }m \to \infty
\end{align}

Hence, since the test error depends continuously on the covariance for regression trained on gaussian data, we can replace $\ba$ by $\mathcal{N}(0, \bSigma_\bb)$ without changing the generalization error. 

Finally, note that $$\bSigma_\bb = \rho_1^2\bW\bSigma\bW^T + \rho_2^2\bI_m$$

\subsubsection{Multiple hidden layers}

Denote the output of the $\ell$-th hidden layer by $\bx^{(\ell)}$, $\ell = 0,\dots L$.  Same as in the case of one hidden layer, we apply Theorem \ref{thm: univ} to $\bx^{(L)}= \phi(\bW^{(L-1)}\bx^{(L-1)})$. Again, same as in the case of one hidden layer, we have $\bbE_{\ba \sim \bx^{(L)}} \ba  = 0$ and denote

\begin{align}\label{eq: recurs}
    & \bSigma_{\ell} = \bbE_{\bx^{(\ell)}}  \bx^{(\ell)} {\bx^{(\ell)}}^T \quad \ell = 0,\dots, L \nonumber \\ 
    & \tbSigma_{\ell} =   \rho_{\ell,1}^2\bW^{(\ell)}\bSigma_{\ell - 1}{\bW^{(\ell)}}^T + \rho_{\ell,2}^2\bI \quad \ell = 1,\dots, L
\end{align}

where \begin{align}
& \rho_{\ell,1} = \frac{1}{\sigma^2_{\ell-1}} \mathbb{E}_{z \sim \mathcal{N}(0, \sigma^2_{\ell-1})}z\phi(z)\\ 
& \rho_{\ell,2}^2 = \mathbb{E}_{z \sim \mathcal{N}(0, \sigma^2_{\ell-1})}\phi(z)^2 - \sigma^2_{\ell-1}\rho_{\ell,1}^2 \\
& \sigma_{\ell}^2 = \frac{\tr(\bSigma_\ell)}{d_\ell}
\end{align}

Similarly to the one hidden layer case, our goal is to show that
\begin{align}\label{eq: approx_cov}
    \|\bSigma_{\ell} - \tbSigma_\ell\|_{op} = o(\frac{1}{d_\ell})
\end{align}

Note that \eqref{eq: approx_cov} provides an asymptotic recurrence for finding $\bSigma^{(L)}$. 

To prove \eqref{eq: approx_cov}, note that, if $\bx^{(\ell-1)}$ were Gaussian, we would be able to obtain the desired result in the same way as for one hidden layer by appealing to the results of \cite{schroder2023deterministic}. However, $\bx^{(\ell-1)}$ is not Gaussian in general for $\ell>1$. As such, we outline a more general argument based on subgaussianity below.  

\begin{definition}
A random variable $s$ is called subgaussian if there exists a constant $\sigma > 0$ such that for all $t \geq 0$,
$$\mathbb{P}(|s| \geq t) \leq 2e^{-\frac{t^2}{2\sigma^2}}.$$
The smallest such constant $\sigma$ is called the subgaussian parameter of $s$.
\end{definition}

We will also need a definition of the Lipschitz Concentration Property:

% \begin{lemma}
%     Each $\bw_i^T\bx^{(\ell)}$ is subgaussian with parameter $\sigma = O(\frac{1}{\sqrt{d}})$ for $\ell = 1,\dots, L$ and $i=1,\dots,m$.
% \end{lemma}

% \begin{proof}

% Recall that the data $\bx^{(0)} = \bx$ is Gaussian with a well-conditioned $\bSigma$ satisfying $\Tr(\bSigma) = O(1)$ by assumption. Also recall that $\bx^{(\ell)} = \phi(\bW^{(\ell)}\bx^{(\ell - 1)})$ by definition. 
    
% \end{proof}

\begin{definition}[Lipschitz Concentration Property]
A random vector $\bz \in \mathbb{R}^n$ satisfies the Lipschitz Concentration Property (LCP) with parameter $\sigma$ if for any $L$-Lipschitz function $f: \mathbb{R}^n \to \mathbb{R}$, the random variable $f(\bz) - \mathbb{E}[f(\bz)]$ is subgaussian with parameter $L\sigma$. That is, for all $t > 0$:
\[
\mathbb{P}(|f(\bz) - \mathbb{E}[f(\bz)]| \geq t) \leq 2\exp\left(-\frac{t^2}{2L^2\sigma^2}\right)
\]
\end{definition}

\begin{remark}
The LCP is preserved under Lipschitz mappings: if $\bz$ satisfies LCP with parameter $\sigma$ and $g: \mathbb{R}^n \to \mathbb{R}^m$ is $L_g$-Lipschitz, then $g(\bz)$ satisfies LCP with parameter $L_g\sigma$.
\end{remark}

We will make use of the following lemma in the rest of the proof: 

\begin{lemma}\label{lm: subg}
    Each $\bw_i^T\bx^{(\ell)}$ is subgaussian with parameter $\sigma = O(\frac{1}{\sqrt{d}})$ for $\ell = 1,\dots, L$ and $i=1,\dots,m$.
\end{lemma}

\begin{proof}

Recall that the data $\bx^{(0)} = \bx$ is Gaussian with a well-conditioned $\bSigma$ satisfying $\Tr(\bSigma) = O(1)$ by assumption. Also recall that $\bx^{(\ell)} = \phi(\bW^{(\ell)}\bx^{(\ell - 1)})$ by definition.

\textbf{Step 1: Initial data satisfies LCP.}
Since $\bx^{(0)} \sim \mathcal{N}(\mathbf{0}, \bSigma)$ with $\Tr(\bSigma) = O(1)$ and $\bSigma$ is well-conditioned, we have:
\begin{itemize}
    \item $\Tr(\bSigma) = \sum_{i=1}^d \lambda_i = O(1)$
    \item Well-conditioned means $\lambda_{\max}/\lambda_{\min} = O(1)$, so all eigenvalues are of the same order
    \item This implies $d \cdot \lambda_{\max} = O(1)$, hence $\lambda_{\max} = O(1/d)$
\end{itemize}
Since Gaussian random vectors satisfy LCP with parameter proportional to $\sqrt{\lambda_{\max}}$, we have that $\bx^{(0)}$ satisfies LCP with parameter $\sigma_0 = O(1/\sqrt{d})$.

\textbf{Step 2: LCP is preserved through layers.}
We proceed by induction on $\ell$. Assume $\bx^{(\ell-1)}$ satisfies LCP with parameter $\sigma_{\ell-1}$.

Consider the mapping $\bx^{(\ell-1)} \mapsto \bx^{(\ell)} = \phi(\bW^{(\ell)}\bx^{(\ell-1)})$. Since:
\begin{itemize}
    \item The linear map $\bx^{(\ell-1)} \mapsto \bW^{(\ell)}\bx^{(\ell-1)}$ is Lipschitz with constant $\|\bW^{(\ell)}\|_{\text{op}}$
    \item The activation function $\phi$ is assumed to be Lipschitz (typically with constant 1 for ReLU, sigmoid, tanh, etc.)
\end{itemize}

The composition $\bx^{(\ell-1)} \mapsto \phi(\bW^{(\ell)}\bx^{(\ell-1)})$ is Lipschitz with constant $L_\phi \cdot \|\bW^{(\ell)}\|_{\text{op}}$.

By the preservation of LCP under Lipschitz mappings, $\bx^{(\ell)}$ satisfies LCP with parameter $\sigma_\ell = L_\phi \cdot \|\bW^{(\ell)}\|_{\text{op}} \cdot \sigma_{\ell-1}$.

\textbf{Step 3: Bounding the LCP parameter.}
Assuming \eqref{eq: generic_weights} holds, we have $\|\bW^{(\ell)}\|_{\text{op}} = O(1)$ with high probability. Thus:
\[
\sigma_\ell = O(1) \cdot \sigma_{\ell-1} = O(1) \cdot O(1/\sqrt{d}) = O(1/\sqrt{d})
\]
for all $\ell \leq L$, maintaining the $O(1/\sqrt{d})$ LCP parameter across layers.

\textbf{Step 4: Linear functionals of LCP vectors.}
For any fixed vector $\bw_i$ with $\|\bw_i\|_2 = O(1)$, the linear functional $f(\bx^{(\ell)}) = \bw_i^T\bx^{(\ell)}$ is $O(1)$-Lipschitz.

Since $\bx^{(\ell)}$ satisfies LCP with parameter $\sigma_\ell = O(1/\sqrt{d})$, we have that $\bw_i^T\bx^{(\ell)}$ is subgaussian with parameter $O(1) \cdot O(1/\sqrt{d}) = O(1/\sqrt{d})$.

We will need another technical lemma as well:

\begin{lemma}
The following decomposition holds for any bounded odd $\phi$ with bounded first, third and fifth derivatives:
    \begin{align}
        \phi(\bw_i^T\bx^{(\ell)}) = \phi'(0)\bw_i^T\bx^{(\ell)}  + \frac{\phi'''(0)}{6}(\bw_i^T\bx^{(\ell)})^3 + O\left(\frac{1}{d^{5/2}}\right) 
    \end{align}
\end{lemma}

\begin{proof}
Since $\phi$ is odd, we have $\phi(0) = 0$ and all even derivatives vanish at 0. By Taylor's theorem with remainder:

\begin{align}
\phi(z) &= \phi'(0)z + \frac{\phi'''(0)}{6}z^3 + \frac{\phi^{(5)}(\xi)}{120}z^5
\end{align}

for some $\xi$ between 0 and $z$.

Setting $z = \bw_i^T\bx^{(\ell)}$:
\begin{align}
\phi(\bw_i^T\bx^{(\ell)}) &= \phi'(0)\bw_i^T\bx^{(\ell)} + \frac{\phi'''(0)}{6}(\bw_i^T\bx^{(\ell)})^3 + R_i
\end{align}

where the remainder term is:
$$R_i = \frac{\phi^{(5)}(\xi)}{120}(\bw_i^T\bx^{(\ell)})^5$$

From our earlier result, $\bw_i^T\bx^{(\ell)}$ is subgaussian with parameter $O(1/\sqrt{d})$. Therefore, with high probability:
$$|\bw_i^T\bx^{(\ell)}| = O(1/\sqrt{d})$$

Thus:
$$|R_i| = \left|\frac{\phi^{(5)}(\xi)}{120}(\bw_i^T\bx^{(\ell)})^5\right| = O\left(\frac{1}{(\sqrt{d})^5}\right) = O\left(\frac{1}{d^{5/2}}\right)$$

\end{proof}

\begin{lemma}\label{lemma: gauss_interm}
Let $\mathbf{g} \sim \mathcal{N}(0,\bSigma^{(\ell)})$. Then the following holds for all $i,j$:
\begin{align}
\left|\mathbb{E}_{\bx^{(\ell)}}[\phi(\bw_i^T\bx^{(\ell)})\phi(\bw_j^T\bx^{(\ell)})] - \mathbb{E}_{\mathbf{g}}[\phi(\bw_i^T\mathbf{g})\phi(\bw_j^T\mathbf{g})]\right| = O(1/d^2)
\end{align}
Moreover, if $i \ne j$, then one has:
\begin{align}
\left|\mathbb{E}_{\bx^{(\ell)}}[\phi(\bw_i^T\bx^{(\ell)})\phi(\bw_j^T\bx^{(\ell)})] - \mathbb{E}_{\mathbf{g}}[\phi(\bw_i^T\mathbf{g})\phi(\bw_j^T\mathbf{g})]\right| = O(1/d^3)
\end{align}

\end{lemma}

\begin{proof}

We will apply Taylor expansion to both terms.

For the subgaussian $\bx^{(\ell)}$:
\begin{align}
\phi(\bw_i^T\bx^{(\ell)})\phi(\bw_j^T\bx^{(\ell)}) &= \left[\phi'(0)\bw_i^T\bx^{(\ell)} + \frac{\phi'''(0)}{6}(\bw_i^T\bx^{(\ell)})^3 + O(1/d^{5/2})\right]\\
&\quad \times \left[\phi'(0)\bw_j^T\bx^{(\ell)} + \frac{\phi'''(0)}{6}(\bw_j^T\bx^{(\ell)})^3 + O(1/d^{5/2})\right]
\end{align}

Expanding:
\begin{align}
= [\phi'(0)]^2(\bw_i^T\bx^{(\ell)})(\bw_j^T\bx^{(\ell)}) &+ \frac{\phi'(0)\phi'''(0)}{6}\left[(\bw_i^T\bx^{(\ell)})(\bw_j^T\bx^{(\ell)})^3 + (\bw_i^T\bx^{(\ell)})^3(\bw_j^T\bx^{(\ell)})\right]\\
&+ \frac{[\phi'''(0)]^2}{36}(\bw_i^T\bx^{(\ell)})^3(\bw_j^T\bx^{(\ell)})^3 + O(1/d^{5/2})
\end{align}

Similarly for the Gaussian $\mathbf{g}$:
\begin{align}
\phi(\bw_i^T\mathbf{g})\phi(\bw_j^T\mathbf{g}) = [\phi'(0)]^2(\bw_i^T\mathbf{g})(\bw_j^T\mathbf{g}) &+ \frac{\phi'(0)\phi'''(0)}{6}\left[(\bw_i^T\mathbf{g})(\bw_j^T\mathbf{g})^3 + (\bw_i^T\mathbf{g})^3(\bw_j^T\mathbf{g})\right]\\
&+ \frac{[\phi'''(0)]^2}{36}(\bw_i^T\mathbf{g})^3(\bw_j^T\mathbf{g})^3 + O(1/d^{5/2})
\end{align}

Now, we compare expectations of each term one ny one.

For the second order term:
\begin{align}
\mathbb{E}[(\bw_i^T\bx^{(\ell)})(\bw_j^T\bx^{(\ell)})] = \bw_i^T\bSigma^{(\ell)}\bw_j = \mathbb{E}[(\bw_i^T\mathbf{g})(\bw_j^T\mathbf{g})]
\end{align}
These match exactly by assumption.

For the forth-order cross-terms like $\mathbb{E}[(\bw_i^T\bx^{(\ell)})(\bw_j^T\bx^{(\ell)})^3]$: both are $O(1/d^2)$ since $\bw_i^T\bx^{(\ell)} = O(1/\sqrt{d})$ and $\bw_j^T\bx^{(\ell)} = O(1/\sqrt{d})$, therefore the difference is at most $O(1/d^2)$.

All remaining terms are $O(1/d^3)$ for both distributions, so we are done with the first part.

For the sharper bound for $i \ne j$, note that the fourth moment expands as:
\[
\mathbb{E}[(\bw_i^T\bx^{(\ell)})(\bw_j^T\bx^{(\ell)})^3] = \sum_{k,\ell,m,n} w_{ik}w_{j\ell}w_{jm}w_{jn}\mathbb{E}[x_k^{(\ell)}x_\ell^{(\ell)}x_m^{(\ell)}x_n^{(\ell)}]
\]

Assuming that $\bSigma_{\ell}$ is diagonal WLOG (rotate it to become diagonal and apply the same rotation to $\bw_i, \bw_j$ if not), we can assume that only terms with paired indices survive. Since each corresponding expectation is $O(\frac{1}{d^2})$, each coefficient $w_{ik}w_{j\ell}w_{jm}w_{jn}$ is of order $O(\frac{1}{d^2})$, there are $O(d^2)$ of these coefficients and $\bw_i$ is independent from $\bw_j$, we conclude that $\mathbb{E}[(\bw_i^T\bx^{(\ell)})(\bw_j^T\bx^{(\ell)})^3]$ is indeed of order $O(\frac{1}{d^3})$ as desired.

\end{proof}

\begin{lemma}
Given the covariance recursions:
\begin{align}
\bSigma_{\ell} &= \mathbb{E}_{\bx^{(\ell)}} [\bx^{(\ell)} (\bx^{(\ell)})^T], \quad \ell = 0,\dots, L \\
\tilde{\bSigma}_{\ell} &= \rho_1^2\bW^{(\ell)}\bSigma_{\ell - 1}(\bW^{(\ell)})^T + \rho_2^2\bI
\end{align}

Suppose that $\|\bSigma_{\ell-1} - \tilde{\bSigma}_{\ell-1}\| = O(\delta_{\ell-1})$ for some $\delta_{\ell-1}$. Then:
\[
\|\bSigma_{\ell} - \tilde{\bSigma}_{\ell}\| = O\left(\frac{m}{d^2} + \rho_1^2 \|\bW^{(\ell)}\|^2 \delta_{\ell-1}\right)
\]
where $m$ is the dimension of $\bx^{(\ell)}$.
\end{lemma}

\begin{proof}
\textbf{Step 1: Express $\bSigma_{\ell}$ in terms of the activation function.}

Since $\bx^{(\ell)} = \phi(\bW^{(\ell)}\bx^{(\ell-1)})$ applied element-wise:
\[
[\bSigma_{\ell}]_{ij} = \mathbb{E}[\phi(\bw_i^T\bx^{(\ell-1)})\phi(\bw_j^T\bx^{(\ell-1)})]
\]
where $\bw_i^T$ is the $i$-th row of $\bW^{(\ell)}$.

\textbf{Step 2: Define intermediate Gaussian covariance.}

Let $\mathbf{g}^{(\ell-1)} \sim \mathcal{N}(0, \bSigma_{\ell-1})$ and define:
\[
\hat{\bSigma}_{\ell} = \mathbb{E}[\hat{\bx}^{(\ell)}(\hat{\bx}^{(\ell)})^T]
\]
where $\hat{x}_i^{(\ell)} = \phi(\bw_i^T\mathbf{g}^{(\ell-1)})$.

By the previous lemma, each entry satisfies:
\[
|[\bSigma_{\ell}]_{ij} - [\hat{\bSigma}_{\ell}]_{ij}| = O(1/d^2)
\]

\textbf{Step 3: Relate $\hat{\bSigma}_{\ell}$ to the Gaussian model.}

For Gaussian inputs $\mathbf{g}^{(\ell-1)}$, using Taylor expansion and Gaussian moment formulas:
\[
[\hat{\bSigma}_{\ell}]_{ij} = \rho_1^2 \bw_i^T\bSigma_{\ell-1}\bw_j + \rho_2^2\delta_{ij} + O(1/d^{3/2})
\]

This can be written as:
\[
\hat{\bSigma}_{\ell} = \rho_1^2\bW^{(\ell)}\bSigma_{\ell-1}(\bW^{(\ell)})^T + \rho_2^2\bI + O(1/d^{3/2})
\]

\textbf{Step 4: Account for the approximation error from the previous layer.}

Since $\|\bSigma_{\ell-1} - \tilde{\bSigma}_{\ell-1}\| = O(\delta_{\ell-1})$:
\begin{align}
\|\hat{\bSigma}_{\ell} - \tilde{\bSigma}_{\ell}\| &= \|\rho_1^2\bW^{(\ell)}\bSigma_{\ell-1}(\bW^{(\ell)})^T - \rho_1^2\bW^{(\ell)}\tilde{\bSigma}_{\ell-1}(\bW^{(\ell)})^T\| + O(1/d^{3/2})\\
&= \rho_1^2 \|\bW^{(\ell)}(\bSigma_{\ell-1} - \tilde{\bSigma}_{\ell-1})(\bW^{(\ell)})^T\| + O(1/d^{3/2})\\
&\leq \rho_1^2 \|\bW^{(\ell)}\|^2 \|\bSigma_{\ell-1} - \tilde{\bSigma}_{\ell-1}\| + O(1/d^{3/2})\\
&= O(\rho_1^2 \|\bW^{(\ell)}\|^2 \delta_{\ell-1})
\end{align}

\textbf{Step 5: Combine bounds using triangle inequality.}

Using the Frobenius norm argument from Step 2:
\[
\|\bSigma_{\ell} - \hat{\bSigma}_{\ell}\|_F^2 = \sum_{i,j=1}^m O(1/d^4) = O(m^2/d^4)
\]

Therefore $\|\bSigma_{\ell} - \hat{\bSigma}_{\ell}\| \leq \|\bSigma_{\ell} - \hat{\bSigma}_{\ell}\|_F = O(m/d^2)$.

Combining with Step 4:
\begin{align}
\|\bSigma_{\ell} - \tilde{\bSigma}_{\ell}\| &\leq \|\bSigma_{\ell} - \hat{\bSigma}_{\ell}\| + \|\hat{\bSigma}_{\ell} - \tilde{\bSigma}_{\ell}\|\\
&= O(m/d^2) + O(\rho_1^2 \|\bW^{(\ell)}\|^2 \delta_{\ell-1})
\end{align}
\end{proof}

\end{proof}

\subsection{Gaussian Universality}\label{sec: app_gaus_univ}

Below we verify that we can apply Theorem \ref{thm: univ} to the outputs of the penultimate layer $\bx^{L}$ of the Random Features under the assumptions made in Subsection \ref{subs: notation}. 

Note that Assumptions $2$,$3$ and $4$ from the list of Assumptions \ref{ass: univ} and explicitly assumed to hold in Subsection \ref{subs: notation} and Theorem \ref{thm: main} and are inevitable if we want to apply Theorem \ref{thm: univ}. 

For Assumption 1 from the list, note that

\begin{enumerate}
    \item The mean of each row is $\bmu = 0$ because $\sigma$ is assumed to be odd and the moments of $\bx^{L}$ are bounded due to the subgaussianity of $\|\bx^{L}\|$, which follows from the LCP property \ref{def: LCP} of $\bx^{L}$ and is proven in the Step 1 of Lemma \ref{lm: subg}.
    \item In particular, $\|\bmu\| = 0 = O(1)$. 
    \item For any fixed vector of bounded norm, $\bv^T\bx^{(L)}$ is subgaussian as $\bx^{(L)}$ with $\sigma = O(\frac{1}{\sqrt{d}})$, as it satisfies the LCP property \ref{def: LCP}, which is proven in the Step 1 of Lemma \ref{lm: subg}. This implies $Var(\bv^T\bx^{(L)}) = O(\frac{1}{d})$. 
    \item Denoting the outputs of the $L$-th hidden layer applied to the training samples $\bx_1,\dots,\bx^{L}$ by $\bX^{L}$, it remais to verify that $\sigma_{\min}(\bX^{(L)}{\bX^{(L)}}^T) = \Omega(1)$. The latter follows from the universality of the Marchenko-Pastur law for data satisfying LCP proven in \cite{pmlr-v119-seddik20a}. 
    
\end{enumerate}

\section{CGMT Derivations} \label{sec: CGMT}

After applying Gaussian Universality and Gaussian Equivalence, we use a framework called Convex Gaussian Min-Max Theorem \cite{thrampoulidis2014tight, akhtiamov2024novel} to derive asymptotically tight expressions for the generalization error of the Random Features trained via SMD with different mirrors. For SGD, we have provided the resulting nonlinear system of scalar equations required to find the generalization error in \eqref{eq: sgd_final}. The case of  general mirrors can be found in \eqref{eq: mirr_final}. Thus, \eqref{eq: mirr_final} is the optimization referred to in  Theorem \ref{thm: main} and \eqref{eq: sgd_final} is its particular case corresponding to the SGD. 

\subsection{SGD}

We consider the training datapoints $\{(\bx_i, y_i)\}_{i=1}^n$ being generated according to the model $y_i = \ba_\ast \phi (\bW \bx_i)$ where $\bx_i \in \bbR^d$, $\ba_\ast \in \bbR^D$  We denote $\bbE \bx_i = 0$ and $\bbE \bx_i \bx_i^T = \bSigma_{L- 1}$. We train $\ba$ using SGD initialized from $\bzero$ by minimizing the squared loss $\sum_{i=1}^n (\ba^T \phi(\bW_{L}\bx_i) - y_i)^2 $. Letting the input matrix $\bX \in \bbR^{n \times d}$ with each row corresponding to $\bx_i$, We know from the implicit bias of SGD:
\begin{align*}
    &\min_{\ba} \|\ba-\bs_0\|_2^2 \\
    s.t \quad \ba^T \phi(\bW_L \bX^T) &= \bY^T = \ba_\ast^T  \phi(\bW_L \bX^T)
\end{align*}
Where $\bY \in \bbR^{n \times 1}$. We define $\ba \leftarrow \ba - \ba_0$. Thus we may write:
\begin{align*}
    \min_{\ba} &\|\ba\|_2^2 \\
    s.t \quad &\bG (\ba - \ba_\ast) = \bzero
\end{align*}
Where for each row of $\bG \in \bbR^{n \times D}$, $\bg_i \in \bbR^D$, we have from the Gaussian Equivalence Principal:
\begin{align*}
    \bbE \bG = \bzero, \quad \bSigma_{L} \approx \rho_{L,1}^2 \bW_L \bSigma_{L-1}\bW_L^T + \rho_{L,2}^2 \bI
\end{align*}
The generalization error is:
\begin{align*}
    g.e := \bbE_\bx \Bigl(\hat{\ba} ^T \phi (\bW_L \bx) - \ba_\ast^T \phi(\bW_L \bx)\Bigr)^2 = \rho_{L,1}^2  (\hat{\ba} - \ba_\ast)^T \bW_L \bSigma_{L-1} \bW_L^T (\hat{\ba} - \ba_\ast) + \rho_{L,2}^2 \|\hat{\ba} - \ba_\ast\|_2^2 
\end{align*}
Now using a Lagrange multiplier, we formulate the optimization as a min-max:
\begin{align*}
    \min_{\ba} \max_{\bv_L} \|\ba\|_2^2 + \bv_L^T \tlbG \bSigma_L^{1/2} (\ba - \ba_\ast)
\end{align*}
Now using CGMT, we obtain:
\begin{align*}
     \min_{\ba} \max_{\bv_L}  \|\bv_L\|_2 \bg^T \bSigma_L^{1/2} (\ba - \ba_\ast) + \|\bSigma_L^{1/2} (\ba - \ba_\ast)\|_2 \bh_{o}^T \bv_L +  \|\ba\|_2^2
\end{align*}
Doing the optimization over the direction of $\bv_L$ yields:
\begin{align*}
     \min_{\ba} \max_{\beta > 0} \beta \bg^T \bSigma_L^{1/2} (\ba - \ba_\ast) + \beta \|\bSigma_L^{1/2} (\ba - \ba_\ast)\|_2 \cdot \|\bh_{o}\|_2  +  \|\ba\|_2^2
\end{align*}
Using the square-root trick $\sqrt{t} = \frac{\tau}{2} + \frac{t}{2\tau}$, we observe:
\begin{align*}
    \min_{\ba} \max_{\beta > 0} \min_{\tau > 0} \beta \bg_{o}^T \bSigma_L^{1/2} (\ba - \ba_\ast) + \frac{\beta \tau}{2} + \frac{\beta}{2 \tau} \|\bSigma_L^{1/2} (\ba - \ba_\ast)\|^2_2 \cdot \|\bh_{o}\|^2_2  +  \|\ba\|_2^2
\end{align*}
Furthermore, we have that $g.e = \tau^2$. We note the convexity and concavity of the objective, hence we may exchange the order of min and max:
\begin{align*}
    \max_{\beta > 0} \min_{\tau > 0}  \frac{\beta \tau}{2}  +  \min_{\ba}  \beta \bg_{o}^T \bSigma_L^{1/2} (\ba - \ba_\ast) + \frac{\beta}{2 \tau} \|\bSigma_L^{1/2} (\ba - \ba_\ast)\|^2_2 \cdot \|\bh_{o}\|^2_2  +  \|\ba\|_2^2
\end{align*}
Now note that
\begin{align*}
    \bSigma_L^{1/2} \bg_{o} \sim \calN\Bigl(\bzero, \rho_{L,1}^2 \bW_L \bSigma_{L-1}\bW_L^T + \rho_{L,2}^2 \bI\Bigr)
\end{align*}
Thus we may write $ \bSigma_L^{1/2} \bg_{o}  = \rho_{L,1} \bW_L \bSigma_{L-1}^{1/2} \tlbg_{L-1,1} +\rho_{L,2} \tlbg_{L-1,2}$ with $\tlbg_{L-1,1} $ and $\tlbg_{L-1,2}$ being independent of each other. Therefore the optimization turns into
\begin{align*}
     \max_{\beta > 0} \min_{\tau > 0}  \frac{\beta \tau}{2}  +  \min_{\ba} & \rho_{L,2}  \beta \tlbg_{L-1,2}^T (\ba - \ba_\ast) + \rho_{L,1} \beta  \tlbg_{L-1,1}^T \bSigma_{L-1}^{1/2} \bW_L^T (\ba - \ba_\ast)  \\ 
     &+ \frac{\rho_{L,1}^2 \beta \|\bh_{o}\|^2_2 }{2 \tau} \Bigl\|\bSigma_{L-1}^{1/2} \bW_L^T (\ba - \ba_\ast)\Bigr\|_2^2  + \frac{\rho_{L,2}^2 \beta \|\bh_{o}\|^2_2 }{2   \tau} \|\ba - \ba_\ast\|^2_2  +  \|\ba\|_2^2
\end{align*}
Now we complete the squares over $\bW_L^T(\ba - \ba_\ast)$.
\begin{align*}
    \max_{\beta > 0} \min_{\tau > 0} & \frac{\beta \tau}{2} \Bigl(1 - \frac{\|\tlbg_{L-1,1}\|_2^2}{\|\bh_{o}\|_2^2} \Bigr)  +  \min_{\ba}  \rho_{L,2} \beta \tlbg_{L-1,2}^T (\ba - \ba_\ast) \\ 
    &+ \frac{\beta}{2 \tau} \Bigl\| \rho_{L,1} \|\bh_{o}\|_2 \cdot \bSigma_{L-1}^{1/2} \bW_L^T (\ba - \ba_\ast) + \frac{\tau}{\|\bh_{o}\|_2} \tlbg_{L-1,1}\Bigr\|_2^2  +  \frac{\rho_{L,2}^2 \beta \|\bh_{o}\|^2_2 }{2   \tau} \|\ba - \ba_\ast\|^2_2  +  \|\ba\|_2^2
\end{align*}
Now we focus on the inner optimization and we use a Fenchel dual to rewrite the quadratic term as
\begin{align*}
    \min_{\ba} \max_{\bu}  & \frac{\beta \rho_{L,1} \|\bh_{o}\|_2  }{2 \tau} \bu^T \bSigma_{L-1}^{1/2} \bW_L^T (\ba - \ba_\ast) +  \frac{\beta}{2\|\bh_{o}\|_2} \bu^T\tlbg_{L-1,1} - \frac{\beta \|\bu\|^2_2}{8\tau}\\ 
    &+ \rho_{L,2} \beta \tlbg_{L-1,2}^T (\ba - \ba_\ast) +\frac{\rho_{L,2}^2 \beta \|\bh_{o}\|^2_2 }{2   \tau} \|\ba - \ba_\ast\|^2_2  +  \|\ba\|_2^2
\end{align*}
Swapping the min and max:
\begin{align*}
     \max_{\bu}  \min_{\ba} & \frac{\beta \rho_{L,1} \|\bh_{o}\|_2  }{2 \tau} \bu^T \bSigma_{L-1}^{1/2} \bW_L^T (\ba - \ba_\ast) +  \frac{\beta}{2\|\bh_{o}\|_2} \bu^T\tlbg_{L-1,1} - \frac{\beta \|\bu\|^2_2}{8\tau}\\ 
     &+\rho_{L,2}\beta \tlbg_{L-1,2}^T (\ba - \ba_\ast) + \frac{\rho_{L,2}^2 \beta \|\bh_{o}\|^2_2 }{2   \tau} \|\ba - \ba_\ast\|^2_2  +  \|\ba\|_2^2
\end{align*}
Employing CGMT again:
\begin{align*}
     \max_{\bu}  \min_{\ba} &\frac{\beta \rho_{L,1} \|\bh_{o}\|_2  }{2 \tau} \|\bSigma_{L-1}^{1/2}\bu\|_2 \bg_{L-1}^T (\ba - \ba_\ast) + \frac{\beta \rho_{L,1} \|\bh_{o}\|_2  }{2 \tau} \|\ba - \ba_\ast\|_2 \bh_{L-1} ^T \bSigma_{L-1}^{1/2} \bu \\ 
     &+  \frac{\beta}{2\|\bh_{o}\|_2} \bu^T\tlbg_{L-1,1} - \frac{\beta \|\bu\|^2_2}{8\tau} +\rho_{L,2} \beta \tlbg_{L-1,2}^T (\ba - \ba_\ast) + \frac{\rho_{L,2}^2 \beta \|\bh_{o}\|^2_2 }{2   \tau} \|\ba - \ba_\ast\|^2_2  +  \|\ba\|_2^2
\end{align*}
Now we perform the optimization over the direction of $\ba - \ba_\ast$. First we observe that 
\begin{align*}
    \|\ba\|_2^2 = \|\ba- \ba_\ast\|_2^2 + 2 \ba_\ast^T (\ba - \ba_\ast) - \|\ba_\ast\|_2^2
\end{align*}
Dropping the constant term $\|\ba_\ast\|_2^2$, we have
\begin{align*}
    \max_{\bu}  \min_{\ba} &\frac{\beta \rho_{L,1} \|\bh_{o}\|_2  }{2 \tau} \|\bSigma_{L-1}^{1/2}\bu\|_2 \bg_{L-1}^T (\ba - \ba_\ast) + \frac{\beta \rho_{L,1} \|\bh_{o}\|_2  }{2 \tau} \|\ba - \ba_\ast\|_2 \bh_{L-1}  ^T \bSigma_{L-1}^{1/2} \bu \\ 
    &+  \frac{\beta}{2\|\bh_{o}\|_2} \bu^T\tlbg_{L-1,1} - \frac{\beta \|\bu\|^2_2}{8\tau} +   \rho_{L,2} \beta \tlbg_{L-1,2}^T (\ba - \ba_\ast) + \Bigl(\frac{\rho_{L,2}^2 \beta \|\bh_{o}\|^2_2 }{2  \tau} + 1 \Bigr)\|\ba - \ba_\ast\|^2_2  +  2 \ba_\ast^T (\ba - \ba_\ast)
\end{align*}
We observe that $\ba - \ba_\ast$ aligns with
\begin{align*}
     \frac{\beta \rho_{L,1} \|\bh_{o}\|_2  }{2 \tau} \|\bSigma_{L-1}^{1/2}\bu\|_2 \bg_{L-1} + \rho_{L,2} \beta \tlbg_{L-1,2} + 2 \ba_\ast
\end{align*}
Thus the optimization turns into
\begin{align*}
    \max_{\bu}  \min_{\eta_{L-1} > 0}  &\frac{\beta \rho_{L,1} \eta_{L-1} \|\bh_{o}\|_2  }{2 \tau} \bh_{L-1}  ^T \bSigma_{L-1}^{1/2} \bu - \eta_{L-1} \Bigl\| \frac{\beta \rho_{L,1} \|\bh_{o}\|_2  }{2 \tau} \|\bSigma_{L-1}^{1/2}\bu\|_2 \bg_{L-1} +  \rho_{L,2} \beta \tlbg_{L-1,2} + 2 \ba_\ast \Bigr\|_2 \\ 
    &+ \frac{\beta}{2\|\bh_{o}\|_2} \bu^T\tlbg_{L-1,1} - \frac{\beta \|\bu\|^2_2}{8\tau}  + \Bigl(\frac{\rho_{L,2}^2\beta \|\bh_{o}\|^2_2 }{2  \tau} + 1 \Bigr) \eta_{L-1}^2  
\end{align*}
Applying the square-root trick again
\begin{align*}
    \max_{\bu}  \min_{\eta_{L-1} > 0} \max_{\alpha_{L-1} > 0} &\frac{\beta \rho_{L,1} \eta_{L-1} \|\bh_{o}\|_2  }{2 \tau} \bh_{L-1}  ^T \bSigma_{L-1}^{1/2} \bu - \frac{\alpha_{L-1} \eta_{L-1}}{2} +  \frac{\beta}{2\|\bh_{o}\|_2} \bu^T\tlbg_{L-1,1} - \frac{\beta \|\bu\|^2_2}{8\tau}  \\
    &+ \Bigl(\frac{\rho_{L,2}^2 \beta \|\bh_{o}\|^2_2 }{2  \tau} + 1 \Bigr) \eta_{L-1}^2 - \frac{\eta_{L-1}}{2 \alpha_{L-1}} \Bigl( \frac{\beta^2 \rho_{L,1}^2 \|\bh_{o}\|^2_2  }{4 \tau^2} \|\bSigma_{L-1}^{1/2}\bu\|^2_2 \|\bg_{L-1}\|_2^2 
    \\&+ \rho^2_1 \beta^2 \|\tlbg_{L-1,2}\|_2^2 + 4 \|\ba_\ast\|_2^2 \Bigr)   
\end{align*}
Using convexity-concavity, we exchange the order of optimizations:
\begin{align*}
     &\min_{\eta_{L-1} > 0} \max_{\alpha_{L-1} > 0}  \Bigl(\frac{\rho_{L,2}^2 \beta \|\bh_{o}\|^2_2 }{2 \tau} + 1 \Bigr) \eta_{L-1}^2     - \frac{\alpha_{L-1} \eta_{L-1}}{2} - \frac{\eta_{L-1}}{2 \alpha_{L-1}} \Bigl(\rho_{L,2}^2 \beta^2 \|\tlbg_{L-1,2}\|_2^2 + 4 \|\ba_\ast\|_2^2 \Bigr) \\ 
     &+\max_{\bu} \frac{\beta}{2\|\bh_{o}\|_2} \bu^T\tlbg_{L-1,1} -  \frac{\eta_{L-1} \beta^2 \rho_{L,1}^2 \|\bh_{o}\|^2_2  }{8 \alpha_{L-1} \tau^2} \|\bSigma_{L-1}^{1/2}\bu\|^2_2 \|\bg_{L-1}\|_2^2  - \frac{\beta \|\bu\|^2_2}{8\tau}   + \frac{\beta \rho_{L,1} \eta_{L-1} \|\bh_{o}\|_2  }{2 \tau} \bh_{L-1}  ^T \bSigma_{L-1}^{1/2} \bu
\end{align*}

We know from the recursion $\bSigma_{L-1} =\rho_{L-1,1}^2 \bW_{L-1} \bSigma_{L-2}\bW_{L-1}^T + \rho_{L-1,2}^2  \bI$. Applying the same technique as before, we take $ \bSigma_{L-1}^{1/2} \bh_{L-1}  = \rho_{L-1,1} \bW \bSigma_{L-2}^{1/2} \tlbh_{L-2, 1} +\rho_{L-1,2}  \tlbh_{L-2, 2}$ and consider the inner optimization
\begin{align*}
    \max_{\bu} &\frac{\beta}{2\|\bh_{o}\|_2} \bu^T\tlbg_{L-1,1} -  \frac{\eta_{L-1} \beta^2 \rho_{L,1}^2 \rho_{L-1,1}^2 \|\bh_{o}\|^2_2  }{8 \alpha_{L-1} \tau^2} \|\bSigma_{L-2}^{1/2}\bW_{L-1}^T \bu\|^2_2 \|\bg_{L-1}\|_2^2 
    -  \frac{\eta_{L-1} \beta^2 \rho_{L,1}^2 \rho_{L-1,2}^2 \|\bh_{o}\|^2_2  }{8 \alpha_{L-1} \tau^2} \|\bu\|^2_2 \|\bg_{L-1}\|_2^2 \\
    & - \frac{\beta \|\bu\|^2_2}{8\tau} 
    + \frac{\beta \rho_{L,1} \rho_{L-1,1} \eta_{L-1} \|\bh_{o}\|_2}{2 \tau} \tlbh_{L-2,1}^T \bSigma_{L-2}^{1/2} \bW_{L-1}^T \bu 
    + \frac{\beta \rho_{L,1} \rho_{L-1,2} \eta_{L-1} \|\bh_{o}\|_2}{2 \tau} \tlbh_{L-2, 2} ^T\bu 
\end{align*}
Completing the squares:
\begin{align*}
    \max_{\bu} & \frac{\beta}{2\|\bh_{o}\|_2} \bu^T\tlbg_{L-1,1} -  \frac{\eta_{L-1}  }{2 \alpha_{L-1}}  \Bigl\| \frac{\beta \rho_{L,1} \rho_{L-1,1} \|\bh_{o}\|_2 \|\bg_{L-1}\|_2}{2 \tau} \bSigma_{L-2}^{1/2}\bW_{L-1}^T \bu - \frac{\alpha_{L-1}}{\|\bg_{L-1}\|_2} \tlbh_{L-2,1}\Bigr\|_2^2 \\ 
    &+ \frac{ \|\tlbg_{L-2,1}\|_2^2}{2\|\bg_{L-1}\|^2_2 } \alpha_{L-1} \eta_{L-1}  -  \frac{\eta_{L-1} \beta^2 \rho_{L,1}^2 \rho_{L-1,2}^2 \|\bh_{o}\|^2_2  }{8 \alpha_{L-1} \tau^2} \|\bu\|^2_2 \|\bg_{L-1}\|_2^2  - \frac{\beta \|\bu\|^2_2}{8\tau}
    + \frac{\beta \rho_{L,1} \rho_{L-1,2} \eta_{L-1} \|\bh_{o}\|_2}{2 \tau} \tlbh_{L-2, 2} ^T\bu 
\end{align*}
We drop the term $\frac{ \|\tlbg_{L-2,1}\|_2^2}{2\|\bg_{L-1}\|^2_2 } \alpha_{L-1} \eta_{L-1}$ from the optimization as it does not depend on $\bu$. Now introducing $\bv_{L-2}$ as the Fenchel dual:
\begin{align*}
    \max_{\bu} \min_{\bv_{L-2}} &\frac{\beta}{2\|\bh_{o}\|_2} \bu^T\tlbg_{L-1,1} -  \frac{\eta_{L-1} }{2 \alpha_{L-1}}   \frac{\beta \rho_{L,1} \rho_{L-1,1}  \|\bh_{o}\|_2 \|\bg_{L-1}\|_2}{2 \tau} \bv_{L-2}^T\bSigma_{L-2}^{1/2}\bW_{L-1}^T\bu + \frac{\eta_{L-1}}{2\|\bg_{L-1}\|_2} \bv_{L-2}^T\tlbh_{L-2,1} \\
    &+ \frac{\eta_{L-1} \|\bv_{L-2}\|_2^2 }{8 \alpha_{L-1}} -  \frac{\eta_{L-1} \beta^2 \rho_{L,1}^2 \rho_{L-1,2}^2 \|\bh_{o}\|^2_2  }{8 \alpha_{L-1} \tau^2} \|\bu\|^2_2 \|\bg_{L-1}\|_2^2  - \frac{\beta \|\bu\|^2_2}{8\tau}
    + \frac{\beta \rho_{L,1} \rho_{L-1,2} \eta_{L-1} \|\bh_{o}\|_2}{2 \tau } \tlbh_{L-2, 2} ^T\bu 
\end{align*}
Exchanging the order of min and max, we then apply CGMT w.r.t $\bW$, obtaining:
\begin{align*}
     \min_{\bv_{L-2}} \max_{\bu}  &\frac{\beta}{2\|\bh_{o}\|_2} \bu^T\tlbg_{L-1,1} 
     - \frac{\eta_{L-1}\beta \rho_{L,1} \rho_{L-1,1}  \|\bh_{o}\|_2 \|\bg_{L-1}\|_2}{4\alpha_{L-1} \tau}  \Bigl( \|\bSigma_{L-2}^{1/2}\bv_{L-2}\|_2 \bh_{L-2}^T \bu + \|\bu\|_2 \bg_{L-2}^T \bSigma_{L-2}^{1/2}\bv_{L-2}  \Bigr)\\
     &+ \frac{\eta_{L-1}}{2\|\bg_{L-1}\|_2} \bv_{L-2}^T\tlbh_{L-2,1} + \frac{\eta_{L-1} \|\bv_{L-2}\|_2^2 }{8 \alpha_{L-1}} -  \frac{\eta_{L-1} \beta^2 \rho_{L,1}^2 \rho_{L-1,2}^2 \|\bh_{o}\|^2_2  }{8 \alpha_{L-1} \tau^2} \|\bu\|^2_2 \|\bg_{L-1}\|_2^2  - \frac{\beta \|\bu\|^2_2}{8\tau}
     \\ 
     &+ \frac{\beta \rho_{L,1} \rho_{L-1,2} \eta_{L-1} \|\bh_{o}\|_2}{2 \tau} \tlbh_{L-2, 2} ^T\bu 
\end{align*}
Doing the optimization over the direction of $\bu$, yields
\begin{align*}
     \min_{\bv_{L-2}} \max_{\eta_{L-2} > 0} &- \frac{\eta_{L-1}\beta \rho_{L,1} \rho_{L-1,1} \eta_{L-1} \|\bh_{o}\|_2 \|\bg_{L-1}\|_2}{4\alpha_{L-1} \tau} \eta_{L-2} \bg_{L-2}^T \bSigma_{L-2}^{1/2}\bv_{L-2}   +\frac{\eta_{L-1}}{2\|\bg_{L-1}\|_2} \bv_{L-2}^T\tlbh_{L-2,1} \\
     &+\eta_{L-2} \Bigl\|  \frac{\beta}{2\|\bh_{o}\|_2} \tlbg_{L-1,1}  - \frac{\eta_{L-1}\beta \rho_{L,1} \rho_{L-1,1} \|\bh_{o}\|_2 \|\bg_{L-1}\|_2}{4\alpha_{L-1} \tau} \|\bSigma_{L-2}^{1/2}\bv_{L-2}\|_2 \bh_{L-2} \\
     & + \frac{\beta \rho_{L,1} \rho_{L-1,2} \eta_{L-1} \|\bh_{o}\|_2}{2 \tau} \tlbh_{L-2, 2}  \Bigr\|_2 
     + \frac{\eta_{L-1} \|\bv_{L-2}\|_2^2 }{8 \alpha_{L-1}}  -  \Bigl(\frac{\eta_{L-1} \beta^2 \rho_{L,1}^2 \rho_{L-1,2}^2 \|\bh_{o}\|^2_2 \|\bg_{L-1}\|_2^2  }{8 \alpha_{L-1} \tau^2}  + \frac{\beta}{8\tau}\Bigr) \eta_{L-2}^2
\end{align*}
Applying the square-root trick again, we obtain:
\begin{align*}
    \min_{\bv_{L-2}} \max_{\eta_{L-2} > 0} \min_{\alpha_{L-2}>0} &\frac{\eta_{L-2} \alpha_{L-2}}{2}  - \frac{\eta_{L-1}\beta \rho_{L,1} \rho_{L-1,1}  \|\bh_{o}\|_2 \|\bg_{L-1}\|_2}{4\alpha_{L-1} \tau}  \eta_{L-2} \bg_{L-2}^T \bSigma_{L-2}^{1/2}\bv_{L-2}   \\
    &+ \frac{\eta_{L-2}}{2\alpha_{L-2}} \Bigl(  \frac{\beta^2}{4\|\bh_{o}\|^2_2} \|\tlbg_{L-1,1}\|_2^2 
    +  \frac{\eta^2_{L-1}\beta^2 \rho^2_{L,1} \rho^2_{L-1,1} \|\bh_{o}\|^2_2 \|\bg_{L-1}\|^2_2}{16\alpha^2_{L-1} \tau^2} \|\bSigma_{L-2}^{1/2}\bv_{L-2}\|^2_2 \|\bh_{L-2}\|_2^2 \\
    &+  \frac{\beta^2 \rho^2_{L,1} \rho^2_{L-1,2} \eta^2_{L-1} \|\bh_{o}\|^2_2}{4 \tau^2 }  \|\tlbh_{L-2,2}\|_2^2  \Bigr) 
    + \frac{\eta_{L-1}}{2\|\bg_{L-1}\|_2} \bv_{L-2}^T\tlbh_{L-2,1}  + \frac{\eta_{L-1}  \|\bv_{L-2}\|_2^2 }{8 \alpha_{L-1}}  
    \\ 
    &-  \Bigl(\frac{\eta_{L-1} \beta^2 \rho_{L,1}^2 \rho_{L-1,2}^2 \|\bh_{o}\|^2_2 \|\bg_{L-1}\|_2^2  }{8 \alpha_{L-1} \tau^2}  + \frac{\beta}{8\tau}\Bigr) \eta_{L-2}^2
\end{align*}

We exchange the orders of min and max because of convexity and concavity
\begin{align*}
    \max_{\eta_{L-2} > 0} \min_{\alpha_{L-2}>0} & \frac{\eta_{L-2} \alpha_{L-2}}{2} 
     -  \Bigl(\frac{\eta_{L-1} \beta^2 \rho_{L,1}^2 \rho_{L-1,2}^2 \|\bh_{o}\|^2_2 \|\bg_{L-1}\|_2^2  }{8 \alpha_{L-1} \tau^2}  + \frac{\beta}{8\tau}\Bigr) \eta_{L-2}^2 \\
      &+ \frac{\eta_{L-2}}{2\alpha_{L-2}} \Bigl(\frac{\beta^2\|\tlbg_{L-1,1}\|_2^2 }{4\|\bh_{o}\|^2_2} +  \frac{\beta^2 \rho^2_{L,1} \rho^2_{L-1,2} \eta^2_{L-1} \|\bh_{o}\|^2_2\|\tlbh_{L-2,2}\|_2^2 }{4 \tau^2 }  \Bigr) \\
     + \min_{\bv_{L-2}} & \frac{\eta_{L-2}}{2\alpha_{L-2}} \frac{\eta^2_{L-1}\beta^2 \rho^2_{L,1} \rho^2_{L-1,1} \|\bh_{o}\|^2_2 \|\bg_{L-1}\|^2_2 \|\bh_{L-2}\|_2^2 }{16\alpha^2_{L-1} \tau^2} \|\bSigma_{L-2}^{1/2}\bv_{L-2}\|^2_2 +  \frac{\eta_{L-1}  \|\bv_{L-2}\|_2^2 }{8 \alpha_{L-1}}  \\
     &- \frac{\eta_{L-1}\beta \rho_{L,1} \rho_{L-1,1}  \|\bh_{o}\|_2 \|\bg_{L-1}\|_2}{4\alpha_{L-1} \tau}  \eta_{L-2} \bg_{L-2}^T \bSigma_{L-2}^{1/2}\bv_{L-2}   + \frac{\eta_{L-1}}{2\|\bg_{L-1}\|_2} \bv_{L-2}^T\tlbh_{L-2,1} 
\end{align*}
Now we use the recursion step $\bSigma_{L-2} =\rho_{L-2,1}^2 \bW_{L-2} \bSigma_{L-3}\bW_{L-2}^T + \rho_{L-2,2}^2 \bI$ and write
\begin{align*}
    \min_{\bv_{L-2}}  &\frac{\eta_{L-2}}{2\alpha_{L-2}} \frac{\eta^2_{L-1}\beta^2 \rho^2_{L,1} \rho^2_{L-1,1} \rho_{L-2,1}^2  \|\bh_{o}\|^2_2 \|\bg_{L-1}\|^2_2 \|\bh_{L-2}\|_2^2 }{16\alpha^2_{L-1} \tau^2} \|\bSigma_{L-3}^{1/2}\bW_{L-2}^T\bv_{L-2}\|^2_2 \\
    &+ \frac{\eta_{L-2}}{2\alpha_{L-2}} \frac{\eta^2_{L-1}\beta^2 \rho^2_{L,1} \rho^2_{L-1,1} \rho_{L-2,2}^2  \|\bh_{o}\|^2_2 \|\bg_{L-1}\|^2_2 \|\bh_{L-2}\|_2^2 }{16\alpha^2_{L-1} \tau^2} \|\bv_{L-2}\|_2^2 \\
     &- \frac{\eta_{L-1}\beta \rho_{L,1} \rho_{L-1,1} \rho_{L-2,1} \|\bh_{o}\|_2 \|\bg_{L-1}\|_2}{4\alpha_{L-1} \tau}  \eta_{L-2} \tlbg_{L-3,1}^T \bSigma_{L-3}^{1/2} \bW_{L-2}^T \bv_{L-2}  \\ 
     & - \frac{\eta_{L-1}\beta \rho_{L,1} \rho_{L-1,1} \rho_{L-2,2} \|\bh_{o}\|_2 \|\bg_{L-1}\|_2}{4\alpha_{L-1} \tau}  \eta_{L-2} \tlbg_{L-3,2}^T  \bv_{L-2}
     + \frac{\eta_{L-1}}{2\|\bg_{L-1}\|_2} \bv_{L-2}^T\tlbh_{L-2,1} 
    + \frac{\eta_{L-1}  \|\bv_{L-2}\|_2^2 }{8 \alpha_{L-1}} 
\end{align*}
Completing the squares yields
\begin{align*}
    \min_{\bv_{L-2}} & \frac{\eta_{L-2}}{2\alpha_{L-2}} \Bigl\|  \frac{\eta_{L-1}\beta \rho_{L,1} \rho_{L-1,1} \rho_{L-2,1}  \|\bh_{o}\|_2 \|\bg_{L-1}\|_2 \|\bh_{L-2}\|_2 }{4\alpha_{L-1} \tau} \bSigma_{L-3}^{1/2}\bW_{L-2}^T\bv_{L-2} - \frac{\alpha_{L-2}}{\|\bh_{L-2}\|_2} \tlbg_{L-3,1}\Bigr\|_2^2 \\
     &- \frac{\alpha_{L-2}\eta_{L-2}}{2} \frac{\|\tlbg_{L-3,1}\|_2^2}{\|\bh_{L-2}\|^2_2}  
     + \Bigl(\frac{\eta_{L-2}}{2\alpha_{L-2}} \frac{\eta^2_{L-1}\beta^2 \rho^2_{L,1} \rho^2_{L-1,1} \rho_{L-2,2}^2  \|\bh_{o}\|^2_2 \|\bg_{L-1}\|^2_2 \|\bh_{L-2}\|_2^2 }{16\alpha^2_{L-1} \tau^2} + \frac{\eta_{L-1}}{8\alpha_{L-1}}\Bigr) \|\bv_{L-2}\|_2^2 \\
     &-\frac{\beta \eta_{L-1}\rho_{L,1} \rho_{L-1,1} \rho_{L-2,2} \|\bh_{o}\|_2 \|\bg_{L-1}\|_2}{4\alpha_{L-1} \tau}  \eta_{L-2} \tlbg_{L-3,2}^T  \bv_{L-2}
     + \frac{\eta_{L-1}}{2\|\bg_{L-1}\|_2} \bv_{L-2}^T\tlbh_{L-2,1} 
\end{align*}
Now we consider the optimization over $\bv_{L-2}$, we observe that, the inner optimization takes a similar form to that what was obtained earlier. So far, we have:
\begin{align*}
     \max_{\beta > 0} \min_{\tau > 0}  & \frac{\beta \tau}{2} \Bigl(1 - \frac{\|\tlbg_{L-1,1}\|_2^2}{\|\bh_{o}\|_2^2} \Bigr) \\
      +\min_{\eta_{L-1} > 0} \max_{\alpha_{L-1} > 0}  &- \frac{\alpha_{L-1} \eta_{L-1}}{2} \Bigl( 1 - \frac{ \|\tlbg_{L-2,1}\|_2^2}{\|\bg_{L-1}\|^2_2 } \Bigr) +  \Bigl(\frac{\rho_{L,2}^2 \beta \|\bh_{o}\|^2_2 }{2 \tau} + 1 \Bigr) \eta_{L-1}^2 - \frac{\eta_{L-1}}{2 \alpha_{L-1}} \Bigl( \beta^2 \rho_{L,2}^2  \|\tlbg_{L-1,2}\|_2^2 + 4 \|\ba_\ast\|_2^2 \Bigr) \\
      +\max_{\eta_{L-2} > 0} \min_{\alpha_{L-2}>0} & \frac{\eta_{L-2} \alpha_{L-2}}{2} \Bigl(1 - \frac{\|\tlbg_{L-3,1}\|_2^2}{\|\bh_{L-2}\|^2_2} \Bigr)
     -  \Bigl(\frac{\eta_{L-1} \beta^2 \rho_{L,1}^2 \rho_{L-1,2}^2 \|\bh_{o}\|^2_2 \|\bg_{L-1}\|_2^2  }{8 \alpha_{L-1} \tau^2}  + \frac{\beta}{8\tau}\Bigr) \eta_{L-2}^2 \\
      &+ \frac{\eta_{L-2}}{2\alpha_{L-2}} \Bigl(\frac{\beta^2\|\tlbg_{L-1,1}\|_2^2 }{4\|\bh_{o}\|^2_2} +  \frac{\beta^2 \rho^2_{L,1} \rho^2_{L-1,2} \eta^2_{L-1} \|\bh_{o}\|^2_2\|\tlbh_{L-2,2}\|_2^2 }{4 \tau^2 }  \Bigr) \\
      +\min_{\bv_{L-2}}  &\frac{\eta_{L-2}}{2\alpha_{L-2}} \Bigl\|  \frac{\eta_{L-1}\beta \rho_{L,1} \rho_{L-1,1} \rho_{L-2,1}  \|\bh_{o}\|_2 \|\bg_{L-1}\|_2 \|\bh_{L-2}\|_2 }{4\alpha_{L-1} \tau} \bSigma_{L-3}^{1/2}\bW_{L-2}^T\bv_{L-2} - \frac{\alpha_{L-2}}{\|\bh_{L-2}\|_2} \tlbg_{L-3,1}\Bigr\|_2^2 \\
     &+   \Bigl(\frac{\eta_{L-2}}{2\alpha_{L-2}} \frac{\eta^2_{L-1}\beta^2 \rho^2_{L,1} \rho^2_{L-1,1} \rho_{L-2,2}^2  \|\bh_{o}\|^2_2 \|\bg_{L-1}\|^2_2 \|\bh_{L-2}\|_2^2 }{16\alpha^2_{L-1} \tau^2} + \frac{\eta_{L-1}}{8\alpha_{L-1}}\Bigr) \|\bv_{L-2}\|_2^2 \\
    & -\frac{\beta \eta_{L-1}\rho_{L,1} \rho_{L-1,1} \rho_{L-2,2} \|\bh_{o}\|_2 \|\bg_{L-1}\|_2}{4\alpha_{L-1} \tau}  \eta_{L-2} \tlbg_{L-3,2}^T  \bv_{L-2} + \frac{\eta_{L-1}}{2\|\bg_{L-1}\|_2} \bv_{L-2}^T\tlbh_{L-2,1} 
\end{align*}
Continuing this process, for the final optimization we have:
\begin{align*}
    \min_{\bv_1}  \frac{\eta_2}{2\alpha_2} \Bigl\|  c_{L} \bSigma_{0}^{1/2}\bW_{1}^T\bv_1 - \frac{\alpha_2}{\|\bh_{1}\|_2} \tlbg_{0,1}\Bigr\|_2^2 
     + \frac{\eta_2}{2\alpha_2} \frac{c_L^2 \rho^2_{1,2}}{ \rho^2_{1,1}} \|\bv_1\|_2^2 
     - \frac{c_L \rho_{1,2}}{ \rho_{1,1} \|\bg_1\|_2}   \eta_2 \tlbg_{0,2}^T  \bv_1
     + \frac{\eta_3}{2\|\bg_{1}\|_2} \bv_1^T\tlbh_{1,1} 
    + \frac{\eta_3  \|\bv_1\|_2^2 }{8 \alpha_3} 
\end{align*}
Using Fenchel dual and introducing $\bv_0$, yields
\begin{align*}
    \min_{\bv_1} \max_{\bv_0}  &\frac{\eta_2}{2\alpha_2}  c_{L} \bv_0^T \bSigma_{0}^{1/2}\bW_{1}^T\bv_1 -  \frac{\eta_2}{2\|\bh_{1}\|_2}    \bv_0^T \tlbg_{0,1}  - \frac{\eta_2 \|\bv_1\|_2^2}{8 \alpha_2}
     \\ 
     &+ \frac{\eta_2}{2\alpha_2} \frac{c_L^2 \rho^2_{1,2}}{ \rho^2_{1,1}} \|\bv_1\|_2^2 
     - \frac{c_L \rho_{1,2}}{ \rho_{1,1} \|\bg_1\|_2}   \eta_2 \tlbg_{0,2}^T  \bv_1
     + \frac{\eta_3}{2\|\bg_{1}\|_2} \bv_1^T\tlbh_{1,1} 
    + \frac{\eta_3  \|\bv_1\|_2^2 }{8 \alpha_3} 
\end{align*}
Applying CGMT
\begin{align*}
    \min_{\bv_1} \max_{\bv_0} & \frac{\eta_2}{2\alpha_2}  c_{L} \|\bSigma_{0}^{1/2}\bv_0\|_2 \bg_0^T\bv_1 +  \frac{\eta_2}{2\alpha_2}  c_{L} \|\bv_1\|_2 \bh_0^T \bSigma_{0}^{1/2}\bv_0 -  \frac{\eta_2}{2\|\bh_{1}\|_2}  \bv_0^T \tlbg_{0,1}  - \frac{\eta_2 \|\bv_0\|_2^2}{8 \alpha_2}
     \\ &+ \frac{\eta_2}{2\alpha_2} \frac{c_L^2 \rho^2_{1,2}}{ \rho^2_{1,1}} \|\bv_1\|_2^2 
     - \frac{c_L \rho_{1,2}}{ \rho_{1,1} \|\bg_1\|_2}   \eta_2 \tlbg_{0,2}^T  \bv_1
     + \frac{\eta_3}{2\|\bg_{1}\|_2} \bv_1^T\tlbh_{1,1} 
    + \frac{\eta_3  \|\bv_1\|_2^2 }{8 \alpha_3} 
\end{align*}
Doing the optimization over $\bv_1$:
\begin{align*}
      \max_{\bv_0} \min_{\eta_1>0} &- \eta_1 \Bigl\| \frac{ c_{L}\eta_2}{2\alpha_2}  \|\bSigma_{0}^{1/2}\bv_0\|_2 \bg_1 -\frac{c_L \rho_{1,2}}{ \rho_{1,1} \|\bg_1\|_2}   \eta_2 \tlbg_{0,2} + \frac{\eta_3}{2\|\bg_{1}\|_2} \tlbh_{1,1} \Bigr\|_2^2
      -   \frac{\eta_2}{2\|\bh_{1}\|_2}   \bv_0^T \tlbg_{0,1}  - \frac{\eta_2 \|\bv_0\|_2^2}{8 \alpha_2}
     \\ 
     &+ \frac{\eta_2}{2\alpha_2}  c_{L} \eta_1 \bh_1^T \bSigma_{0}^{1/2}\bv_0  + \frac{\eta_2}{2\alpha_2} \frac{c_L^2 \rho^2_{1,2}}{ \rho^2_{1,1}} \eta_1^2
     + \frac{\eta_3  \eta_1^2}{8 \alpha_3}
\end{align*}
Now using the square-root trick, we have
\begin{align*}
    \max_{\bv_0} \min_{\eta_1>0} \max_{\alpha_1>0} &- \frac{\alpha_1 \eta_1}{2} - \frac{\eta_1}{2\alpha_1}  \Bigl( \frac{c^2_{L}\eta^2_2}{4\alpha^2_2}  \|\bSigma_{0}^{1/2}\bv_0\|^2_2 \|\bg_1\|_2^2 + \frac{c_L^2 \rho^2_{1,2}}{ \rho^2_{1,1} \|\bg_1\|^2_2}   \eta^2_2 \|\tlbg_{0,2}\|_2^2 + \frac{\eta^2_3}{4\|\bg_{1}\|^2_2} \|\tlbh_{1,1}\|_2^2 \Bigr)
     \\ 
     &- \frac{\eta_2}{2\|\bh_{1}\|_2}   \bv_0^T \tlbg_{0,1}  - \frac{\eta_2 \|\bv_0\|_2^2}{8 \alpha_2}
      + \frac{\eta_2}{2\alpha_2}  c_{L} \eta_1 \bh_1^T \bSigma_{0}^{1/2}\bv_0  + \frac{\eta_2}{2\alpha_2} \frac{c_L^2 \rho^2_{1,2}}{ \rho^2_{1,1}} \eta_1^2
     + \frac{\eta_3  \eta_1^2}{8 \alpha_3}
\end{align*}
Swapping min, max, we complete the squares over $\bv_0$:
\begin{align*}
    \begin{pmatrix}
        \bv_0^T & 1
    \end{pmatrix}
    \begin{pmatrix}
         \frac{\eta_1}{2\alpha_1} \frac{c^2_{L}\eta^2_2 \|\bg_1\|_2^2}{4\alpha^2_2} \bSigma_0 + \frac{\eta_2}{8 \alpha_2} \bI &  - \frac{\eta_2}{4\|\bh_{1}\|_2}   \tlbg_{0,1} + \frac{\eta_2}{4\alpha_2}  c_{L} \eta_1 \bSigma_{0}^{1/2}\bh_1 \\
        - \frac{\eta_2}{4\|\bh_{1}\|_2}   \tlbg^T_{0,1} + \frac{\eta_2}{4\alpha_2}  c_{L} \eta_1 \bh_1^T\bSigma_{0}^{1/2} & 0
    \end{pmatrix}
    \begin{pmatrix}
        \bv_0 \\
        1
    \end{pmatrix}
\end{align*}
Which yields the scalar optimization
\begin{align*}
    \min_{\eta_1>0} \max_{\alpha_1>0} &- \frac{\alpha_1 \eta_1}{2} - \frac{\eta_1}{2\alpha_1}  \Bigl( \frac{c_L^2 \rho^2_{1,2}}{ \rho^2_{1,1} \|\bg_1\|^2_2}   \eta^2_2 \|\tlbg_{0,2}\|_2^2 + \frac{\eta^2_3}{4\|\bg_{1}\|^2_2} \|\tlbh_{1,1}\|_2^2 \Bigr) + \Bigl(\frac{\eta_3}{8 \alpha_3} + \frac{\eta_2}{2\alpha_2} \frac{c_L^2 \rho^2_{1,2}}{ \rho^2_{1,1}} \Bigr) \eta_1^2
     \\ 
     &+  \Bigl( - \frac{\eta_2}{4\|\bh_{1}\|_2}   \tlbg_{0,1} + \frac{\eta_2}{4\alpha_2}  c_{L} \eta_1 \bSigma_{0}^{1/2}\bh_1 \Bigr)^T \Bigl(  \frac{\eta_1}{2\alpha_1} \frac{c^2_{L}\eta^2_2 \|\bg_1\|_2^2}{4\alpha^2_2} \bSigma_0 + \frac{\eta_2}{8 \alpha_2} \bI  \Bigr)^{-1} \Bigl( - \frac{\eta_2}{4\|\bh_{1}\|_2}   \tlbg_{0,1} + \frac{\eta_2}{4\alpha_2}  c_{L} \eta_1 \bSigma_{0}^{1/2}\bh_1 \Bigr)
\end{align*}
By Hanson-Wright's inequality, we have that
\begin{align*}
     \Bigl( - \frac{\eta_2}{4\|\bh_{1}\|_2}   \tlbg_{0,1} + \frac{\eta_2}{4\alpha_2}  c_{L} \eta_1 \bSigma_{0}^{1/2}\bh_1 \Bigr)^T \Bigl( \frac{\eta_1}{2\alpha_1}  \frac{c^2_{L}\eta^2_2 \|\bg_1\|_2^2}{4\alpha^2_2} \bSigma_0 + \frac{\eta_2}{8 \alpha_2} \bI  \Bigr)^{-1} \Bigl( - \frac{\eta_2}{4\|\bh_{1}\|_2}   \tlbg_{0,1} + \frac{\eta_2}{4\alpha_2}  c_{L} \eta_1 \bSigma_{0}^{1/2}\bh_1 \Bigr) \\
     \rarrowp \frac{\eta_2^2}{16\|\bh_1\|_2^2} \tr \Bigl( \frac{c^2_{L}\eta^2_2 \|\bg_1\|_2^2}{4\alpha^2_2} \bSigma_0 + \frac{\eta_2}{8 \alpha_2} \bI  \Bigr)^{-1} + \frac{\eta^2_2 c^2_L \eta_1^2}{16 \alpha_2^2} \tr \Bigl( \frac{c^2_{L}\eta^2_2 \|\bg_1\|_2^2}{4\alpha^2_2} \bSigma_0 + \frac{\eta_2}{8 \alpha_2} \bI  \Bigr)^{-1} \bSigma_0
\end{align*}
Thus the last scalar optimization would be:
\begin{align*}
     \min_{\eta_1>0} \max_{\alpha_1>0} &- \frac{\alpha_1 \eta_1}{2} - \frac{\eta_1}{2\alpha_1}  \Bigl( \frac{c_L^2 \rho^2_{1,2}}{ \rho^2_{1,1} \|\bg_1\|^2_2}   \eta^2_2 \|\tlbg_{0,2}\|_2^2 + \frac{\eta^2_3}{4\|\bg_{1}\|^2_2} \|\tlbh_{1,1}\|_2^2 \Bigr) + \Bigl(\frac{\eta_3}{8 \alpha_3} + \frac{\eta_2}{2\alpha_2} \frac{c_L^2 \rho^2_{1,2}}{ \rho^2_{1,1}} \Bigr) \eta_1^2
     \\ 
     &+  \frac{\eta_2^2}{16\|\bh_1\|_2^2} \tr \Bigl(  \frac{\eta_1}{2\alpha_1} \frac{c^2_{L}\eta^2_2 \|\bg_1\|_2^2}{4\alpha^2_2} \bSigma_0 + \frac{\eta_2}{8 \alpha_2} \bI  \Bigr)^{-1} + \frac{\eta^2_2 c^2_L \eta_1^2}{16 \alpha_2^2} \tr \Bigl( \frac{\eta_1}{2\alpha_1}  \frac{c^2_{L}\eta^2_2 \|\bg_1\|_2^2}{4\alpha^2_2} \bSigma_0 + \frac{\eta_2}{8 \alpha_2} \bI  \Bigr)^{-1} \bSigma_0
\end{align*}
The final optimization is as follows:
\begin{align*}
     \max_{\beta > 0} \min_{\tau > 0}  & \frac{\beta \tau}{2} \Bigl(1 - \frac{\|\tlbg_{L-1,1}\|_2^2}{\|\bh_{o}\|_2^2} \Bigr) \\
      +\min_{\eta_{L-1} > 0} \max_{\alpha_{L-1} > 0}&  - \frac{\alpha_{L-1} \eta_{L-1}}{2} \Bigl( 1 - \frac{ \|\tlbg_{L-2,1}\|_2^2}{\|\bg_{L-1}\|^2_2 } \Bigr) +  \Bigl(\frac{\rho_{L,2}^2 \beta \|\bh_{o}\|^2_2 }{2 \tau} + 1 \Bigr) \eta_{L-1}^2 - \frac{\eta_{L-1}}{2 \alpha_{L-1}} \Bigl( \beta^2 \rho_{L,2}^2  \|\tlbg_{L-1,2}\|_2^2+ 4 \|\ba_\ast\|_2^2 \Bigr) \\
      + \max_{\eta_{L-2} > 0} \min_{\alpha_{L-2}>0} & \frac{\eta_{L-2} \alpha_{L-2}}{2} \Bigl(1 - \frac{\|\tlbg_{L-3,1}\|_2^2}{\|\bh_{L-2}\|^2_2} \Bigr)-  \Bigl(\frac{\eta_{L-1} \beta^2 \rho_{L,1}^2 \rho_{L-1,2}^2 \|\bh_{o}\|^2_2 \|\bg_{L-1}\|_2^2  }{8 \alpha_{L-1} \tau^2}  + \frac{\beta}{8\tau}\Bigr) \eta_{L-2}^2 \\
      &+ \frac{\eta_{L-2}}{2\alpha_{L-2}} \Bigl(\frac{\beta^2\|\tlbg_{L-1,1}\|_2^2 }{4\|\bh_{o}\|^2_2} +  \frac{\beta^2 \rho^2_{L,1} \rho^2_{L-1,2} \eta^2_{L-1} \|\bh_{o}\|^2_2\|\tlbh_{L-2,2}\|_2^2 }{4 \tau^2 }  \Bigr) \\
      \hdots \\
      +\min_{\eta_1>0} \max_{\alpha_1>0} &- \frac{\alpha_1 \eta_1}{2} - \frac{\eta_1}{2\alpha_1}  \Bigl( \frac{c_L^2 \rho^2_{1,2}}{ \rho^2_{1,1} \|\bg_1\|^2_2}   \eta^2_2 \|\tlbg_{0,2}\|_2^2 + \frac{\eta^2_3}{4\|\bg_{1}\|^2_2} \|\tlbh_{1,1}\|_2^2 \Bigr) + \Bigl(\frac{\eta_3}{8 \alpha_3} + \frac{\eta_2}{2\alpha_2} \frac{c_L^2 \rho^2_{1,2}}{ \rho^2_{1,1}} \Bigr) \eta_1^2
     \\ 
     &+  \frac{\eta_2^2}{16\|\bh_1\|_2^2} \tr \Bigl(  \frac{\eta_1}{2\alpha_1} \frac{c^2_{L}\eta^2_2 \|\bg_1\|_2^2}{4\alpha^2_2} \bSigma_0 + \frac{\eta_2}{8 \alpha_2} \bI  \Bigr)^{-1} + \frac{\eta^2_2 c^2_L \eta_1^2}{16 \alpha_2^2} \tr \Bigl( \frac{\eta_1}{2\alpha_1}  \frac{c^2_{L}\eta^2_2 \|\bg_1\|_2^2}{4\alpha^2_2} \bSigma_0 + \frac{\eta_2}{8 \alpha_2} \bI  \Bigr)^{-1} \bSigma_0
\end{align*}
We take derivative with respect to $\alpha_1$:
\begin{align*}
     \frac{\partial}{\partial \alpha_1} = 0 \Rightarrow 0 &= - \frac{\eta_1}{2} + \frac{\eta_1}{\alpha_1^2}  \Bigl( \frac{c_L^2 \rho^2_{1,2}}{ \rho^2_{1,1} \|\bg_1\|^2_2}   \eta^2_2 \|\tlbg_{0,2}\|_2^2 + \frac{\eta^2_3}{4\|\bg_{1}\|^2_2} \|\tlbh_{1,1}\|_2^2 \Bigr) \\
     &+ \frac{\eta_2^2}{16\|\bh_1\|_2^2} \frac{\eta_1}{2\alpha_1^2} \frac{c^2_{L}\eta^2_2 \|\bg_1\|_2^2}{4\alpha^2_2} \tr \Bigl(  \frac{\eta_1}{2\alpha_1} \frac{c^2_{L}\eta^2_2 \|\bg_1\|_2^2}{4\alpha^2_2} \bSigma_0 + \frac{\eta_2}{8 \alpha_2} \bI  \Bigr)^{-2} \bSigma_0 \\ 
     &+\frac{\eta^2_2 c^2_L \eta_1^2}{16 \alpha_2^2} \frac{\eta_1}{2\alpha^2_1}  \frac{c^2_{L}\eta^2_2 \|\bg_1\|_2^2}{4\alpha^2_2}  \tr \Bigl( \frac{\eta_1}{2\alpha_1}  \frac{c^2_{L}\eta^2_2 \|\bg_1\|_2^2}{4\alpha^2_2} \bSigma_0 + \frac{\eta_2}{8 \alpha_2} \bI  \Bigr)^{-2} \bSigma_0^2
\end{align*}
For $\eta_1$ and $\alpha_1$, after taking derivatives we observe that:
\begin{align*}
    0 = - \alpha_1 + 2 \Bigl(\frac{\eta_3}{8 \alpha_3} + 2 \frac{\eta_2}{2\alpha_2} \frac{c_L^2 \rho^2_{1,2}}{ \rho^2_{1,1}} \Bigr) \eta_1 + 2 \eta_1  \frac{\eta^2_2 c^2_L}{16 \alpha_2^2} \tr \Bigl( \frac{\eta_1}{2\alpha_1}  \frac{c^2_{L}\eta^2_2 \|\bg_1\|_2^2}{4\alpha^2_2} \bSigma_0 + \frac{\eta_3}{8 \alpha_3} \bI  \Bigr)^{-1} \bSigma_0
\end{align*}
Hence, denoting $\zeta_1 = \frac{\eta_1}{\alpha_1}$, we observe that
\begin{align*}
     \zeta_1 = \frac{ 1 -   \zeta_1 \zeta^2_2 c^2_L \tr \Bigl(c^2_{L}\zeta_1\zeta^2_2 \|\bg_1\|_2^2 \bSigma_0 +\zeta_2 \bI  \Bigr)^{-1} }{\frac{\zeta_3}{4} + \frac{2c_L^2 \zeta_2 \rho^2_{1,2}}{ \rho^2_{1,1}} }
\end{align*}
Now we can find $\alpha_1$ by:
\begin{align*}
    \alpha_1^2 = \frac{ \Bigl(\frac{2c_L^2 \rho^2_{1,2}}{ \rho^2_{1,1} \|\bg_1\|^2_2}   \|\tlbg_{0,2}\|_2^2  + \frac{c^2_{L}\zeta^2_2 \|\bg_1\|_2^2 }{\|\bh_1\|_2^2}   \tr \Bigl(  \zeta_1 c^2_{L}\zeta^2_2 \|\bg_1\|_2^2 \bSigma_0 + \zeta_2 \bI \Bigr)^{-2} \bSigma_0 \Bigr)\zeta_2^2 \alpha_2^2 +  \frac{\|\tlbh_{1,1}\|_2^2 }{2\|\bg_{1}\|^2_2 } \zeta_3^2 \alpha_3^2 }{1 -   \zeta^4_2 c^4_L \zeta_1^2 \|\bg_1\|_2^2 \tr \Bigl(   \zeta_1 c^2_{L}\zeta^2_2 \|\bg_1\|_2^2 \bSigma_0 + \zeta_2 \bI  \Bigr)^{-2} \bSigma_0^2}
\end{align*}
For $\eta_2, \alpha_2$ we observe that
\begin{align*}
     \max_{\eta_{2} > 0} \min_{\alpha_{2}>0} & \frac{\eta_{2} \alpha_{2}}{2} \Bigl(1 - \frac{\|\tlbg_{1,1}\|_2^2}{\|\bh_{2}\|^2_2} \Bigr)
     -  \Bigl(\frac{\eta_4}{8 \alpha_4} + \frac{\eta_3}{2\alpha_3} \frac{c_{L-1}^2 \rho^2_{2,2}}{ \rho^2_{2,1}} \Bigr) \eta_{2}^2 
      + \frac{\eta_{2}}{2\alpha_{2}} \Bigl(\frac{\eta_4^2\|\tlbg_{L-1,1}\|_2^2 }{4\|\bh_{o}\|^2_2} + \eta_3^2  \frac{c_{L-1}^2 \rho^2_{2,2}}{ \rho^2_{2,1}} \Bigr)\\
      +\min_{\eta_1>0} \max_{\alpha_1>0} &- \frac{\alpha_1 \eta_1}{2} - \frac{\eta_1}{2\alpha_1}  \Bigl( \frac{c_L^2 \rho^2_{1,2}}{ \rho^2_{1,1} \|\bg_1\|^2_2}   \eta^2_2 \|\tlbg_{0,2}\|_2^2 + \frac{\eta^2_3}{4\|\bg_{1}\|^2_2} \|\tlbh_{1,1}\|_2^2 \Bigr) + \Bigl(\frac{\eta_3}{8 \alpha_3} + \frac{\eta_2}{2\alpha_2} \frac{c_L^2 \rho^2_{1,2}}{ \rho^2_{1,1}} \Bigr) \eta_1^2
     \\ 
     &+  \frac{\eta_2^2}{16\|\bh_1\|_2^2} \tr \Bigl(  \frac{\eta_1}{2\alpha_1} \frac{c^2_{L}\eta^2_2 \|\bg_1\|_2^2}{4\alpha^2_2} \bSigma_0 + \frac{\eta_2}{8 \alpha_2} \bI  \Bigr)^{-1} + \frac{\eta^2_2 c^2_L \eta_1^2}{16 \alpha_2^2} \tr \Bigl( \frac{\eta_1}{2\alpha_1}  \frac{c^2_{L}\eta^2_2 \|\bg_1\|_2^2}{4\alpha^2_2} \bSigma_0 + \frac{\eta_2}{8 \alpha_2} \bI  \Bigr)^{-1} \bSigma_0
\end{align*}
Which yields:
\begin{align*}
     \zeta_2 &= \frac{1-\frac{1}{\|\bh_1\|_2^2} \tr \Bigl(   c^2_{L}\zeta_2 \|\bg_1\|_2^2 \zeta_1 \bSigma_0 + \bI  \Bigr)^{-1} }{\frac{2\zeta_3 c_{L-1}^2 \rho^2_{2,2}}{ 2\rho^2_{2,1}} + \frac{\zeta_4}{4}} \\
     \alpha_2 &= \frac{\frac{c_L^2 \rho^2_{1,2}}{ \rho^2_{1,1}} \zeta_1^2 \alpha_1^2 + F_1 }{1- \frac{d_2}{d_1}}
\end{align*}
Where  
\begin{align}\label{eq: F_one}
    F_1 := &- \frac{\zeta_2^2 \alpha_2^2}{\|\bh_1\|_2^2} \biggl[  c^2_{L}\zeta^2_2 \zeta_1 \|\bg_1\|_2^2  \tr \Bigl( c^2_{L}\zeta^2_2 \zeta_1 \|\bg_1\|_2^2 \bSigma_0 + \zeta_2 \bI  \Bigr)^{-2}\bSigma_0 +  \tr \Bigl( c^2_{L}\eta^2_2 \zeta_1 \|\bg_1\|_2^2 \bSigma_0 + \zeta_2 \bI  \Bigr)^{-2} \biggr] \nonumber \\
    &+ \zeta_2 c^2_L \zeta_1^2 \alpha_1^2\tr \Bigl( c^2_{L}\zeta^2_2 \zeta_1 \|\bg_1\|_2^2 \bSigma_0 + \zeta_2 \bI  \Bigr)^{-1}\bSigma_0 \nonumber \\
    &- \zeta_2^2 c_L^2 \zeta_1^2 \alpha_1^2  \biggl[ c^2_{L}\zeta^2_2 \zeta_1 \|\bg_1\|_2^2  \tr \Bigl( c^2_{L}\zeta^2_2 \zeta_1 \|\bg_1\|_2^2 \bSigma_0 + \zeta_2 \bI  \Bigr)^{-2}\bSigma^2_0 +  \tr \Bigl( c^2_{L}\zeta^2_2 \zeta_1 \|\bg_1\|_2^2 \bSigma_0 + \zeta_2 \bI  \Bigr)^{-2}\bSigma^2\biggr]
\end{align}
Overall, we observe that, for each pair of $(\eta_i, \alpha_i)$ the inner optimization is a function of $\frac{\eta_i}{\alpha_i}$ for $i>2$ and also $\eta_i^2$ for $i=2$. As demonstrated above, to find the generalization error, we find $\zeta_i$'s first and then solve the linear system of equations in terms of $\alpha_i^2$. This implies that for $i > 2$:
\begin{align*}
    \alpha_i^2 &= \frac{\frac{c_{L-i+1}^2 \rho^2_{i,2}}{ \rho^2_{i,1} \|\bg_i\|^2_2}   \zeta^2_{i+1} \alpha_{i+1}^2 \|\tlbg_{i-1,2}\|_2^2 + \frac{\eta^2_{i+2}}{4\|\bg_{i}\|^2_2} \|\tlbh_{i,1}\|_2^2 + \frac{F'}{\zeta_i^2}  + \sum_{j=1}^i \frac{c_{L-j}^2 \rho^2_{j,2}}{ \rho^2_{j,1}}  \zeta_j^2 \alpha_j^2}{1- \frac{d_{i+1}}{d_i} } \\
    \zeta_i &= \frac{1- \frac{d_{i+1}}{d_i}}{\frac{c_{L-i}^2 \rho^2_{i,2} \zeta_{i+1}}{ \rho^2_{i,1}} + \frac{\zeta_{i+2}}{8}}
\end{align*}
With 
\begin{align}\label{eq: F_prim}
    F':= \frac{\partial}{\partial x} \biggl[ \frac{\zeta_2^2 \alpha_2^2}{\|\bh_1\|_2^2} \tr \Bigl( c^2_{L} x^2 \zeta_1 \|\bg_1\|_2^2 \bSigma_0 + \zeta_2 \bI  \Bigr)^{-2}\bSigma_0 +  \zeta^2_2 c^2_L \zeta_1^2 \alpha_1^2 \tr \Bigl( c^2_{L} x^2 \zeta_1 \|\bg_1\|_2^2 \bSigma_0 + \zeta_2 \bI  \Bigr)^{-2} \biggr]
\end{align}
Finally, for $\beta, \tau$, we would have:
\begin{align*}
    \theta = \frac{1- \frac{d_{L}}{d_{L-1}}}{\zeta_{L-2} \frac{d_1}{d_0} +\zeta_{L-1}\rho_{L,2}^2  \|\tlbg_{L-1,2}\|_2^2} \\
    \tau^2 = \frac{ \frac{F'}{\zeta_i^2} + \sum_{j=1}^{L-1} \frac{c_{L-j}^2 \rho^2_{j,2}}{ \rho^2_{j,1}}  \zeta_j^2 \alpha_j^2}{1- \frac{d_L}{d_{L-1}}}
\end{align*}
Summarizing, we have
\begin{align}\label{eq: sgd_final}
     \zeta_1 &= \frac{ 1 -   \zeta_1 \zeta^2_2 c^2_L \tr \Bigl(c^2_{L}\zeta_1\zeta^2_2  \bSigma_0 +\zeta_2 \bI  \Bigr)^{-1} }{\frac{\zeta_3}{4} + \frac{2c_L^2 \zeta_2 \rho^2_{1,2}}{ \rho^2_{1,1}}}, \quad
      \zeta_2 = \frac{1-\frac{1}{d_1} \tr \Bigl(   c^2_{L}\zeta_2 \zeta_1 \bSigma_0 + \bI  \Bigr)^{-1} }{\frac{2\zeta_3 c_{L-1}^2 \rho^2_{2,2}}{ 2\rho^2_{2,1}} + \frac{\zeta_4}{4}} \nonumber \\
      \zeta_i &= \frac{1- \frac{d_{i+1}}{d_i}}{\frac{c_{L-i}^2 \rho^2_{i,2} \zeta_{i+1}}{ \rho^2_{i,1}} + \frac{\zeta_{i+2}}{8}}, \quad i = 3, \cdots, L-1, \quad \theta = \frac{1- \frac{d_{L}}{d_{L-1}}}{\zeta_{L-2} \frac{d_1}{d_0} +\zeta_{L-1}\rho_{L,2}^2 d_{L-1}} \nonumber\\
    \alpha_1^2 &= \frac{ \Bigl(\frac{2c_L^2 \rho^2_{1,2}d_0}{ \rho^2_{1,1} d_1}  + \frac{c^2_{L}\zeta^2_2 }{d_1}   \tr \Bigl(  \zeta_1 c^2_{L}\zeta^2_2 \bSigma_0 + \zeta_2 \bI \Bigr)^{-2} \bSigma_0 \Bigr)\zeta_2^2 \alpha_2^2 +  \frac{d_1 }{2d_0 } \zeta_3^2 \alpha_3^2 }{1 - \zeta^4_2 c^4_L \zeta_1^2 \tr \Bigl(   \zeta_1 c^2_{L}\zeta^2_2\bSigma_0 + \zeta_2 \bI  \Bigr)^{-2} \bSigma_0^2} \nonumber\\
    \alpha^2_2 &= \frac{\frac{c_L^2 \rho^2_{1,2}}{ \rho^2_{1,1}} \zeta_1^2 \alpha_1^2 + F_1 }{1- \frac{d_2}{d_1}} \nonumber \\
    \alpha_i^2 &= \frac{\frac{c_{L-i+1}^2 \rho^2_{i,2}d_{i-1}}{ \rho^2_{i,1} d_{i+1}}   \zeta^2_{i+1} \alpha_{i+1}^2 + \frac{\eta^2_{i+2} d_i}{4d_{i+1}} + \frac{F'}{\zeta_i^2}  + \sum_{j=1}^i \frac{c_{L-j}^2 \rho^2_{j,2}}{ \rho^2_{j,1}}  \zeta_j^2 \alpha_j^2}{1- \frac{d_{i+1}}{d_i} }, \quad i=3,\cdots, L-1 \nonumber \\
    \tau^2 &= \frac{ \frac{F'}{\zeta_i^2} + \sum_{j=1}^{L-1} \frac{c_{L-j}^2 \rho^2_{j,2}}{ \rho^2_{j,1}}  \zeta_j^2 \alpha_j^2}{1- \frac{d_L}{d_{L-1}}} 
\end{align}
Where $F$ and $F'$ are defined in \eqref{eq: F_one} and \eqref{eq: F_prim}, respectively. And we let $c_j := \prod_{\ell=L-j+1}^L d_\ell \zeta_\ell $. To find $\tau^2$, the generalization error, first we find $\zeta_i$'s and $\theta$ through the nonlinear equations described above. Note that these equations are only in terms of $\zeta_i$'s and $\theta$. Then we proceed to solve the linear system of equations in $\alpha_i^2$ and $\tau^2$ to find the generalization error.
\subsection{General mirrors}
After completing the analysis of the case of $\psi = \|\cdot\|_2^2$ in the previous section, we show that for the general case, we can use the results from the previous section. For that, consider the following optimization:
\begin{align*}
    \min_{\ba} \; & D_{\psi}(\ba, \ba_0) \\
    s.t \quad &\bG (\ba - \ba_\ast) = \bzero
\end{align*}
We have Now using a Lagrange multiplier, we the optimization as min-max:
\begin{align*}
    \min_{\ba} \max_{\bv_L} D_{\psi}(\ba, \ba_0) + \bv_L^T \tlbG \bSigma_L^{1/2} (\ba - \ba_\ast)
\end{align*}
Now using CGMT, we obtain:
\begin{align*}
     \min_{\ba} \max_{\bv_L}  \|\bv_L\|_2 \bg^T \bSigma_L^{1/2} (\ba - \ba_\ast) + \|\bSigma_L^{1/2} (\ba - \ba_\ast)\|_2 \bh_{o}^T \bv_L +  D_{\psi}(\ba, \ba_0)
\end{align*}
Doing the optimization over the direction of $\bv_L$ yields:
\begin{align*}
     \min_{\ba} \max_{\beta > 0} \beta \bg^T \bSigma_L^{1/2} (\ba - \ba_\ast) + \beta \|\bSigma_L^{1/2} (\ba - \ba_\ast)\|_2 \cdot \|\bh_{o}\|_2  +  D_{\psi}(\ba, \ba_0)
\end{align*}
Using the square-root trick $\sqrt{t} = \frac{\tau}{2} + \frac{t}{2\tau}$, we observe:
\begin{align*}
    \min_{\ba} \max_{\beta > 0} \min_{\tau > 0} \beta \bg_{o}^T \bSigma_L^{1/2} (\ba - \ba_\ast) + \frac{\beta \tau}{2} + \frac{\beta}{2 \tau} \|\bSigma_L^{1/2} (\ba - \ba_\ast)\|^2_2 \cdot \|\bh_{o}\|^2_2  +  D_{\psi}(\ba, \ba_0)
\end{align*}
Furthermore, we have that $g.e = \tau^2$. We note the convexity and concavity of the objective, hence we may exchange the order of min and max:
\begin{align*}
    \max_{\beta > 0} \min_{\tau > 0}  \frac{\beta \tau}{2}  +  \min_{\ba}  \beta \bg_{o}^T \bSigma_L^{1/2} (\ba - \ba_\ast) + \frac{\beta}{2 \tau} \|\bSigma_L^{1/2} (\ba - \ba_\ast)\|^2_2 \cdot \|\bh_{o}\|^2_2  +  D_{\psi}(\ba, \ba_0)
\end{align*}
Now note that
\begin{align*}
    \bSigma_L^{1/2} \bg_{o} \sim \calN\Bigl(\bzero, \rho_{L,1}^2 \bW_L \bSigma_{L-1}\bW_L^T + \rho_{L,2}^2 \bI\Bigr)
\end{align*}
Thus we may write $ \bSigma_L^{1/2} \bg_{o}  = \rho_{L,1} \bW_L \bSigma_{L-1}^{1/2} \tlbg_{L-1,1} + \rho_{L,2} \tlbg_{L-1,2}$ with $\tlbg_{L-1,1} $ and $\tlbg_{L-1,2}$ being independent of each other. Therefore the optimization turns into
\begin{align*}
     \max_{\beta > 0} \min_{\tau > 0}  &\frac{\beta \tau}{2}  +  \min_{\ba}  \rho_{L,2}  \beta \tlbg_{L-1,2}^T (\ba - \ba_\ast) + \rho_{L,1} \beta  \tlbg_{L-1,1}^T \bSigma_{L-1}^{1/2} \bW_L^T (\ba - \ba_\ast)  \\ 
     &+ \frac{\rho_{L,1}^2 \beta \|\bh_{o}\|^2_2 }{2 \tau} \Bigl\|\bSigma_{L-1}^{1/2} \bW_L^T (\ba - \ba_\ast)\Bigr\|_2^2  + \frac{\rho_{L,2}^2 \beta \|\bh_{o}\|^2_2 }{2   \tau} \|\ba - \ba_\ast\|^2_2  +  D_{\psi}(\ba, \ba_0)
\end{align*}
Now we complete the square over $\bW_L^T(\ba - \ba_\ast)$.
\begin{align*}
    \max_{\beta > 0} \min_{\tau > 0}  &\frac{\beta \tau}{2} \Bigl(1 - \frac{\|\tlbg_{L-1,1}\|_2^2}{\|\bh_{o}\|_2^2} \Bigr)  +  \min_{\ba}  \rho_{L,2} \beta \tlbg_{L-1,2}^T (\ba - \ba_\ast) \\ 
    &+ \frac{\beta}{2 \tau} \Bigl\| \rho_{L,1} \|\bh_{o}\|_2 \cdot \bSigma_{L-1}^{1/2} \bW_L^T (\ba - \ba_\ast) + \frac{\tau}{\|\bh_{o}\|_2} \tlbg_{L-1,1}\Bigr\|_2^2  +  \frac{\rho_{L,2}^2 \beta \|\bh_{o}\|^2_2 }{2   \tau} \|\ba - \ba_\ast\|^2_2  +  D_{\psi}(\ba, \ba_0)
\end{align*}
Now we focus on the inner optimization and we use a Fenchel dual to rewrite the quadratic term as
\begin{align*}
    \min_{\ba} \max_{\bu}   &\frac{\beta \rho_{L,1} \|\bh_{o}\|_2  }{2 \tau} \bu^T \bSigma_{L-1}^{1/2} \bW_L^T (\ba - \ba_\ast) +  \frac{\beta}{2\|\bh_{o}\|_2} \bu^T\tlbg_{L-1,1} - \frac{\beta \|\bu\|^2_2}{8\tau}\\
    &+  \rho_{L,2} \beta \tlbg_{L-1,2}^T (\ba - \ba_\ast) +\frac{\rho_{L,2}^2 \beta \|\bh_{o}\|^2_2 }{2   \tau} \|\ba - \ba_\ast\|^2_2  +  D_{\psi}(\ba, \ba_0)
\end{align*}
Employing CGMT again:
\begin{align*}
      \min_{\ba} \max_{\bu} & \frac{\beta \rho_{L,1} \|\bh_{o}\|_2  }{2 \tau} \|\bSigma_{L-1}^{1/2}\bu\|_2 \bg_{L-1}^T (\ba - \ba_\ast) + \frac{\beta \rho_{L,1} \|\bh_{o}\|_2  }{2 \tau} \|\ba - \ba_\ast\|_2 \bh_{L-1} ^T \bSigma_{L-1}^{1/2} \bu \\ 
      &+  \frac{\beta}{2\|\bh_{o}\|_2} \bu^T\tlbg_{L-1,1} - \frac{\beta \|\bu\|^2_2}{8\tau} + \rho_{L,2} \beta \tlbg_{L-1,2}^T (\ba - \ba_\ast) + \frac{\rho_{L,2}^2 \beta \|\bh_{o}\|^2_2 }{2   \tau} \|\ba - \ba_\ast\|^2_2  +  D_{\psi}(\ba, \ba_0)
\end{align*}
We use the change of variable $\tlbu := \bSigma_{L-1}^{1/2}\bu$ and use the Lagrange multiplier to bring in the constraints:
\begin{align*}
      \min_{\ba}  \min_{\bv_L} & \bv_L^T (\tlbu - \bSigma_{L-1}^{1/2} \bu) + \frac{\beta \rho_{L,1} \|\bh_{o}\|_2  }{2 \tau} \|\tlbu\|_2 \bg_{L-1}^T (\ba - \ba_\ast) + \frac{\beta \rho_{L,1} \|\bh_{o}\|_2  }{2 \tau} \|\ba - \ba_\ast\|_2 \bh_{L-1} ^T \tlbu \\
      &+  \frac{\beta}{2\|\bh_{o}\|_2} \bu^T\tlbg_{L-1,1} - \frac{\beta \|\bu\|^2_2}{8\tau} +\rho_{L,2} \beta \tlbg_{L-1,2}^T (\ba - \ba_\ast) + \frac{\rho_{L,2}^2 \beta \|\bh_{o}\|^2_2 }{2   \tau} \|\ba - \ba_\ast\|^2_2  +  D_{\psi}(\ba, \ba_0)
\end{align*}
Now we perform the optimization over $\bu, \tlbu$ and obtain
\begin{align*}
      \min_{\ba, \bv_L} \max_{\eta_L, \tleta_L} & \tleta_L \Bigl\| \bv_L +  \frac{\beta \rho_{L,1} \|\bh_{o}\|_2  }{2 \tau} \|\ba - \ba_\ast\|_2 \bh_{L-1} \Bigr\|_2 + \eta_L \Bigl\| \bSigma_{L-1}^{1/2} \bv_{L} + \frac{\beta}{2\|\bh_{o}\|_2} \tlbg_{L-1,1}  \Bigr\|_2 +  D_{\psi}(\ba, \ba_0) \\
      &+ \frac{\beta \rho_{L,1} \|\bh_{o}\|_2  }{2 \tau} \tleta_L \bg_{L-1}^T (\ba - \ba_\ast) - \frac{\beta \eta^2_L}{8\tau} + \rho_{L,2} \beta \tlbg_{L-1,2}^T (\ba - \ba_\ast) + \frac{\rho_{L,2}^2 \beta \|\bh_{o}\|^2_2 }{2   \tau} \|\ba - \ba_\ast\|^2_2  
\end{align*}
Now using the square-root trick again:
\begin{align*}
    \min_{\ba, \bv_L} \max_{\eta_L, \tleta_L} \min_{\alpha_L > 0, \tlalpha_L > 0} & \frac{\tleta_L \tlalpha_L}{2} + \frac{\tleta_L}{2\tlalpha_L} \Bigl\| \bv_L +  \frac{\beta \rho_{L,1} \|\bh_{o}\|_2  }{2 \tau} \|\ba - \ba_\ast\|_2 \bh_{L-1} \Bigr\|^2_2 \\
   & + \frac{\alpha_L \eta_L}{2} + \frac{\eta_L}{2\alpha_L}\Bigl\| \bSigma_{L-1}^{1/2} \bv_{L} + \frac{\beta}{2\|\bh_{o}\|_2} \tlbg_{L-1,1}  \Bigr\|^2_2  + \frac{\beta \rho_{L,1} \|\bh_{o}\|_2  }{2 \tau} \tleta_L \bg_{L-1}^T (\ba - \ba_\ast) \\
       &- \frac{\beta \eta^2_L}{8\tau} + \rho_{L,2} \beta \tlbg_{L-1,2}^T (\ba - \ba_\ast) + \frac{\rho_{L,2}^2 \beta \|\bh_{o}\|^2_2 }{2   \tau} \|\ba - \ba_\ast\|^2_2  +  D_{\psi}(\ba, \ba_0)
\end{align*}
Using the recursion,  $\bSigma_{L-1} =\rho_{L-1,1}^2 \bW_{L-1} \bSigma_{L-2}\bW_{L-1}^T + \rho_{L-1,2}^2 \bI$. Applying the same technique as before, we take $ \bSigma_{L-1}^{1/2} \tlbg_{L-1,1}  = \rho_{L-1,1} \bW \bSigma_{L-2}^{1/2} \tlbg_{L-2, 1} + \rho_{L-1,2}  \tlbg_{L-2, 2}$, we can write:
\begin{align*}
    \Bigl\| \bSigma_{L-1}^{1/2} \bv_{L} + \frac{\beta}{2\|\bh_{o}\|_2} \tlbg_{L-1,1}  \Bigr\|^2_2 = \Bigl\| \rho_{L-1,1} \bSigma_{L-2}^{1/2} \bW_{L-2}^T \bv_L + \frac{\beta}{2\|\bh_{o}\|_2} \tlbg_{L-2,1} \Bigr\|_2^2 \\ + \rho_{L-1,2}^2  \|\bv_L\|_2^2 +  \frac{\beta^2 \rho_{L-1,2}^2 }{4\|\bh_{o}\|_2^2} \|\tlbg_{L-2,2}\|_2^2 + \rho_{L-1,2}  \tlbg^T_{L-2, 2} \bv_L
\end{align*}
Plugging back in, we only consider the optimization over $\ba, \bv_L$:
\begin{align*}
     \min_{\ba} &\frac{\tleta_L}{2\tlalpha_L}   \frac{\beta \rho_{L,1} \|\bh_{o}\|_2  }{2 \tau} \tleta_L \bg_{L-1}^T (\ba - \ba_\ast)  + \rho_{L,2}  \beta \tlbg_{L-1,2}^T (\ba - \ba_\ast) + \frac{\rho_{L,2}^2 \beta \|\bh_{o}\|^2_2 }{2   \tau} \|\ba - \ba_\ast\|^2_2  +  D_{\psi}(\ba, \ba_0) \\ 
      + \min_{\bv_L} & \frac{\eta_L}{2\alpha_L} \Bigl\| \rho_{L-1,1} \bSigma_{L-2}^{1/2} \bW_{L-2}^T \bv_L + \frac{\beta}{2\|\bh_{o}\|_2} \tlbg_{L-2,1} \Bigr\|_2^2 + \frac{\eta_L}{2\alpha_L} \rho_{L-1,2}  \tlbg^T_{L-2, 2} \bv_L  \\
      &+ \frac{\eta_L}{2\alpha_L} \rho_{L-1,2}^2   \|\bv_L\|_2^2 +  \frac{\tleta_L}{2\tlalpha_L} \Bigl\| \bv_L +  \frac{\beta \rho_{L,1} \|\bh_{o}\|_2  }{2 \tau} \|\ba - \ba_\ast\|_2 \bh_{L-1} \Bigr\|^2_2 
\end{align*}
Focusing on the inner optimization, we introduce the Fenchel dual variable $\bu_{L-1}$. We note that from here on, the procedure is similar to that of SGD analysis ($\psi = \|\cdot\|_2^2$)
\begin{align*}
    \min_{\bv_L} \max_{\bu_{L-1}} & \frac{\eta_L}{2\alpha_L}  \rho_{L-1,1} \bu_{L-1}^T\bSigma_{L-2}^{1/2} \bW_{L-2}^T \bv_L -   \frac{\eta_L \|\bu_{L-1}\|_2^2}{8\alpha_L} + \frac{\beta}{2\|\bh_{o}\|_2}  \bu_{L-1}^T\tlbg_{L-2,1}  + \frac{\eta_L}{2\alpha_L} \rho_{L-1,2} \tlbg^T_{L-2, 2} \bv_L  \\
      &+ \frac{\eta_L}{2\alpha_L} \rho_{L-1,2}^2  \|\bv_L\|_2^2 +  \frac{\tleta_L}{2\tlalpha_L} \Bigl\| \bv_L +  \frac{\beta \rho_{L,1} \|\bh_{o}\|_2  }{2 \tau} \|\ba - \ba_\ast\|_2 \bh_{L-1} \Bigr\|^2_2 
\end{align*}
Applying CGMT yields,
\begin{align*}
    \min_{\bv_L} \max_{\bu_{L-1}} & \frac{\eta_L \rho_{L-1,1}}{2\alpha_L} \|\bv_L\|_2 \bg_{L-2}^T\bSigma_{L-2}^{1/2} \bu_{L-1} + \frac{\eta_L \rho_{L-1,1}}{2\alpha_L}  \|\bSigma_{L-2}^{1/2} \bu_{L-1}\|_2 \bh_{L-2}^T \bv_L \\ 
    &- \frac{\eta_L \|\bu_{L-1}\|_2^2}{8\alpha_L} + \frac{\beta}{2\|\bh_{o}\|_2}  \bu_{L-1}^T\tlbg_{L-2,1}  + \frac{\eta_L}{2\alpha_L} \rho_{L-1,2} \tlbg^T_{L-2, 2} \bv_L  \\
      &+ \frac{\eta_L}{2\alpha_L} \rho_{L-1,2}^2   \|\bv_L\|_2^2 +  \frac{\tleta_L}{2\tlalpha_L} \Bigl\| \bv_L +  \frac{\beta \rho_{L,1} \|\bh_{o}\|_2  }{2 \tau} \|\ba - \ba_\ast\|_2 \bh_{L-1} \Bigr\|^2_2
\end{align*}
We perform the optimization over the direction of $\bv_L$, we have
\begin{align*}
     \max_{\bu_{L-1}} \min_{\eta_{L-1}} &\frac{\eta_L \rho_{L-1,1}}{2\alpha_L}\eta_{L-1} \bg_{L-2}^T\bSigma_{L-2}^{1/2} \bu_{L-1}  \\
     &-\eta_{L-1}\Bigl\| \frac{\eta_L \rho_{L-1,1}}{2\alpha_L}  \|\bSigma_{L-2}^{1/2} \bu_{L-1}\|_2 \bh_{L-2}  +  \frac{\eta_L}{2\alpha_L} \rho_{L-1,2}  \tlbg_{L-2, 2}  +  \frac{\tleta_L}{\tlalpha_L}\frac{\beta \rho_{L,1} \|\bh_{o}\|_2  }{2 \tau} \|\ba - \ba_\ast\|_2 \bh_{L-1}\Bigr\|_2\\
    &- \frac{\eta_L \|\bu_{L-1}\|_2^2}{8\alpha_L} + \frac{\beta}{2\|\bh_{o}\|_2}  \bu_{L-1}^T\tlbg_{L-2,1} 
      + \Bigl(\frac{\eta_L}{2\alpha_L} \rho_{L-1,2}^2 + \frac{\tleta_L}{2\tlalpha_L} \Bigr) \eta^2_{L-1} +  \frac{\tleta_L}{2\tlalpha_L}  \frac{\beta \|\ba - \ba_\ast\|^2_2  \rho^2_{L,1} \|\bh_{o}\|^2_2 \|\bh^2_{L-1} \|_2 }{4 \tau^2} 
\end{align*}
Using the square-root trick again
\begin{align*}
     \max_{\bu_{L-1}} \min_{\eta_{L-1}} \max_{\alpha_{L-1}} &\frac{\eta_L \rho_{L-1,1}}{2\alpha_L}\eta_{L-1} \bg_{L-2}^T\bSigma_{L-2}^{1/2} \bu_{L-1} - \frac{\alpha_{L-1}\eta_{L-1}}{2} \\
     &-\frac{\eta_{L-1}}{2\alpha_{L-1}}\Bigl(\frac{\eta_L^2 \rho^2_{L-1,1} \| \bh_{L-2}\|_2}{4\alpha^2_L}  \|\bSigma_{L-2}^{1/2} \bu_{L-1}\|^2_2  +  \frac{\eta_L^2}{4\alpha_L^2} \rho_{L-1,2}^2 \|\tlbg_{L-2, 2}\|_2^2   \\
     &+  \frac{\tleta^2_L}{\tlalpha^2_L}\frac{\beta^2 \|\ba - \ba_\ast\|^2_2  \rho^2_{L,1} \|\bh_{o}\|^2_2 \|\bh^2_{L-1} \|_2 }{4 \tau^2} \Bigr) - \frac{\eta_L \|\bu_{L-1}\|_2^2}{8\alpha_L} + \frac{\beta}{2\|\bh_{o}\|_2}  \bu_{L-1}^T\tlbg_{L-2,1} 
      \\
      &+ \Bigl(\frac{\eta_L}{2\alpha_L} \rho_{L-1,2}^2 + \frac{\tleta_L}{2\tlalpha_L} \Bigr) \eta^2_{L-1} +  \frac{\tleta_L}{2\tlalpha_L}  \frac{\beta^2 \|\ba - \ba_\ast\|^2_2  \rho^2_{L,1} \|\bh_{o}\|^2_2 \|\bh^2_{L-1} \|_2 }{4 \tau^2} 
\end{align*}
Now similar to before, we repeatedly apply CGMT and arrive at the following optimization:
\begin{align*}
    \max_{\beta > 0} \min_{\tau > 0}  & \frac{\beta \tau}{2} \Bigl(1 - \frac{\|\tlbg_{L-1,1}\|_2^2}{\|\bh_{o}\|_2^2} \Bigr) \\
    +\max_{\eta_L, \tleta_L} \min_{\alpha_L > 0, \tlalpha_L > 0} & \frac{\tleta_L \tlalpha_L}{2} + \frac{\alpha_L \eta_L}{2} - \frac{\beta \eta^2_L}{8\tau} + \frac{\eta_L}{2\alpha_L} \frac{\beta^2 \rho_{L-1,2}^2 \|\tlbg_{L-2,2}\|_2^2 }{4\|\bh_{o}\|_2^2}  \\
    +  \min_{\ba} & (\ba - \ba_\ast)^T \Bigl( \frac{\tleta_L}{2\tlalpha_L}   \frac{\beta \rho_{L,1} \|\bh_{o}\|_2  }{2 \tau} \tleta_L \bg_{L-1} + \rho_{L,2}  \beta \tlbg_{L-1,2} \Bigr)
      +  \frac{\rho_{L,2}^2 \beta \|\bh_{o}\|^2_2 }{2   \tau} \|\ba - \ba_\ast\|^2_2  \\ 
      &+   \frac{\tleta_L}{2\tlalpha_L}  \frac{\beta^2 \|\ba - \ba_\ast\|^2_2  \rho^2_{L,1} \|\bh_{o}\|^2_2 \|\bh^2_{L-1} \|_2 }{4 \tau^2}    +  D_{\psi}(\ba, \ba_0) + F\Bigl( \|\ba-\ba_\ast\|^2_2 \frac{\beta^2}{\tau^2}, \frac{\eta_L}{\alpha_L}, \frac{\tleta_L}{\tlalpha_L}, \beta\Bigr)
\end{align*}
Where $F$ is defined as
\begin{align}\label{eq: F}
    \max_{\eta_{L-1} > 0} \min_{\alpha_{L-1}>0} &- \Bigl(1- \frac{d_{L-2}}{d_{L-3}} \Bigr)\frac{\alpha_{L-1}\eta_{L-1}}{2} + \Bigl(\frac{\eta_L}{2\alpha_L} \rho_{L-1,2}^2 + \frac{\tleta_L}{2\tlalpha_L} \Bigr) \eta^2_{L-1} \nonumber \\
    &- \frac{\eta_{L-1}}{2\alpha_{L-1}}\Bigl(  \frac{\eta_L^2}{4\alpha_L^2}\rho_{L-1,2}^2 d_{L-2} +\frac{\tleta^2_L}{\tlalpha^2_L}\frac{\beta^2 \|\ba - \ba_\ast\|^2_2  \rho^2_{L,1} n d_{L-1}}{4 \tau^2} \Bigr) \nonumber \\
    \hdots +\min_{\eta_1>0} \max_{\alpha_1>0} &- \frac{\alpha_1 \eta_1}{2} - \frac{\eta_1}{2\alpha_1}  \Bigl( \frac{c_L^2 \rho^2_{1,2} d_0}{ \rho^2_{1,1} d_1}   \eta^2_2 + \frac{\eta^2_2 d_1}{4d_2} \Bigr) + \Bigl(\frac{\eta_3}{8 \alpha_3} + \frac{\eta_2}{2\alpha_2} \frac{c_L^2 \rho^2_{1,2}}{ \rho^2_{1,1}} \Bigr) \eta_1^2
     \nonumber \\ 
     &+  \frac{\eta_2^2}{16\|\bh_1\|_2^2} \tr \Bigl(  \frac{\eta_1}{2\alpha_1} \frac{c^2_{L}\eta^2_2}{4\alpha^2_2} \bSigma_0 + \frac{\eta_2}{8 \alpha_2} \bI  \Bigr)^{-1} + \frac{\eta^2_2 c^2_L \eta_1^2}{16 \alpha_2^2} \tr \Bigl( \frac{\eta_1}{2\alpha_1}  \frac{c^2_{L}\eta^2_2 }{4\alpha^2_2} \bSigma_0 + \frac{\eta_2}{8 \alpha_2} \bI  \Bigr)^{-1} \bSigma_0
\end{align}
% \begin{align}\label{eq: F}
%     \max_{\eta_{L-1} > 0} \min_{\alpha_{L-1}>0} &- \Bigl(1- \frac{\|\bg_{L-2}\|_2^2}{\|\bh_{L-2}^2} \Bigr)\frac{\alpha_{L-1}\eta_{L-1}}{2} + \Bigl(\frac{\eta_L}{2\alpha_L} \rho_{L-1,2}^2 + \frac{\tleta_L}{2\tlalpha_L} \Bigr) \eta^2_{L-1} \nonumber \\
%     &- \frac{\eta_{L-1}}{2\alpha_{L-1}}\Bigl(  \frac{\eta_L^2}{4\alpha_L^2}\rho_{L-1,2}^2 \|\tlbg_{L-2, 2}\|_2^2 +\frac{\tleta^2_L}{\tlalpha^2_L}\frac{\beta^2 \|\ba - \ba_\ast\|^2_2  \rho^2_{L,1} \|\bh_{o}\|^2_2 \|\bh^2_{L-1} \|_2 }{4 \tau^2} \Bigr) \nonumber \\
%     \hdots +\min_{\eta_1>0} \max_{\alpha_1>0} &- \frac{\alpha_1 \eta_1}{2} - \frac{\eta_1}{2\alpha_1}  \Bigl( \frac{c_L^2 \rho^2_{1,2}}{ \rho^2_{1,1} \|\bg_1\|^2_2}   \eta^2_2 \|\tlbg_{0,2}\|_2^2 + \frac{\eta^2_2}{4\|\bg_{1}\|^2_2} \|\tlbh_{1,1}\|_2^2 \Bigr) + \Bigl(\frac{\eta_3}{8 \alpha_3} + \frac{\eta_2}{2\alpha_2} \frac{c_L^2 \rho^2_{1,2}}{ \rho^2_{1,1}} \Bigr) \eta_1^2
%      \nonumber \\ 
%      &+  \frac{\eta_2^2}{16\|\bh_1\|_2^2} \tr \Bigl(  \frac{\eta_1}{2\alpha_1} \frac{c^2_{L}\eta^2_2 \|\bg_1\|_2^2}{4\alpha^2_2} \bSigma_0 + \frac{\eta_2}{8 \alpha_2} \bI  \Bigr)^{-1} + \frac{\eta^2_2 c^2_L \eta_1^2}{16 \alpha_2^2} \tr \Bigl( \frac{\eta_1}{2\alpha_1}  \frac{c^2_{L}\eta^2_2 \|\bg_1\|_2^2}{4\alpha^2_2} \bSigma_0 + \frac{\eta_2}{8 \alpha_2} \bI  \Bigr)^{-1} \bSigma_0
% \end{align}
Using the Lagrange multiplier $\lambda$, we set $\xi:= \|\ba-\ba_\ast\|_2^2$ and we perform the optimization over $\ba$ by completing the squares and obtain that
\begin{align*}
    \min_{\ba} & (\ba - \ba_\ast)^T \Bigl( \frac{\tleta_L}{2\tlalpha_L}   \frac{\beta \rho_{L,1} \|\bh_{o}\|_2  }{2 \tau} \tleta_L \bg_{L-1} + \rho_{L,2}  \beta \tlbg_{L-1,2} \Bigr)
      +  \frac{\rho_{L,2}^2 \beta \|\bh_{o}\|^2_2 }{2   \tau} \|\ba - \ba_\ast\|^2_2  \\ 
      &+   \frac{\tleta_L}{2\tlalpha_L}  \frac{\beta^2 \|\ba - \ba_\ast\|^2_2  \rho^2_{L,1} \|\bh_{o}\|^2_2 \|\bh^2_{L-1} \|_2 }{4 \tau^2}    +  D_{\psi}(\ba, \ba_0) - \lambda \|\ba-\ba_\ast\|_2^2\\
      & \rarrowp d_L \bbE \calM_{\psi;c}\Bigl(a_\ast - \nabla \psi(a_0) - c z\Bigr) + c d_L \bbE z^2 + d_L \bbE (\psi(a_0) - a_0\nabla\psi(a_0)  )
\end{align*}
Where
\begin{align*}
    \calM_{\psi;c}(\cdot) &:= \min_x \frac{1}{2c} (\cdot - x)^2 + \psi(x) \\
    c &:=  \frac{\rho_{L,2}^2 \beta \|\bh_{o}\|^2_2 }{ \tau} + \frac{\tleta_L}{\tlalpha_L}  \frac{\beta^2  \rho^2_{L,1} \|\bh_{o}\|^2_2 \|\bh^2_{L-1} \|_2 }{4 \tau^2} - \lambda \\
    z &\sim \calN\Bigl(0, \frac{\tleta^2_L}{16\tlalpha^2_L}   \frac{\beta^2 \rho^2_{L,1} \|\bh_{o}\|^2_2  }{ \tau^2} \tleta^2_L  + \rho^2_{L,2}  \beta^2\Bigr)
\end{align*}
Hence the final scalar optimization would be
\begin{align}\label{eq: mirr_final}
      \max_{\beta > 0} \min_{\tau > 0}  & \frac{\beta \tau}{2} \Bigl(1 - \frac{d_{L-1}}{n} \Bigr) 
    +\max_{\eta_L, \tleta_L} \min_{\alpha_L > 0, \tlalpha_L > 0}  \frac{\tleta_L \tlalpha_L}{2} + \frac{\alpha_L \eta_L}{2} - \frac{\beta \eta^2_L}{8\tau} + \frac{\eta_L}{2\alpha_L} \frac{\beta^2 \rho_{L-1,2}^2 d_{L-2} }{4n} \nonumber \\
    +  \max_{\lambda} \min_{\xi} &\lambda \xi^2 + F\Bigl( \xi^2 \frac{\beta^2}{\tau^2}, \frac{\eta_L}{\alpha_L}, \frac{\tleta_L}{\tlalpha_L}, \beta\Bigr) + d_L \bbE \calM_{\psi;c}\Bigl(a_\ast - \nabla \psi(a_0) - c z\Bigr) \nonumber \\
    &+ d_L \Bigl(\frac{\rho_{L,2}^2 \beta n }{ \tau} + \frac{\tleta_L}{\tlalpha_L}  \frac{\beta^2  \rho^2_{L,1} n d_{L-1}}{4 \tau^2} - \lambda\Bigr) \cdot \Bigl(  \frac{\tleta^2_L}{16\tlalpha^2_L}   \frac{\beta^2 \rho^2_{L,1} n }{ \tau^2} \tleta^2_L  + \rho^2_{L,2}  \beta^2\Bigr)
\end{align}

Where $F$ is defined in \eqref{eq: F}.

\end{document}